\documentclass{article}

\usepackage{arxiv}

\usepackage[utf8]{inputenc} % allow utf-8 input
\usepackage[T1]{fontenc}    % use 8-bit T1 fonts
\usepackage{hyperref}       % hyperlinks
\usepackage{url}            % simple URL typesetting
\usepackage{booktabs}       % professional-quality tables
\usepackage{amsfonts}       % blackboard math symbols
\usepackage{nicefrac}       % compact symbols for 1/2, etc.
\usepackage{microtype}      % microtypography
\usepackage{lipsum}

\usepackage{times}
\usepackage{epsfig}
\usepackage{graphicx}
\usepackage{amsmath}
\usepackage{amssymb}
\usepackage{makecell}

\usepackage{comment}

\usepackage{multirow}

\usepackage{caption}
\usepackage{amsthm}
\newtheorem{thm}{Theorem}

\newtheorem{lem}{Lemma}

\usepackage{bbding}
\usepackage{pifont}
\usepackage{wasysym}
\usepackage{amssymb}
\usepackage{tablefootnote}
\usepackage{authblk}
\usepackage{tabu}

\newcommand{\etal}{\textit{et al}. }
\newcommand{\ie}{\textit{i}.\textit{e}., }

\usepackage[normalem]{ulem}

\newcommand{\x}{\mathbf{x}}
\newcommand{\X}{\mathbf{X}}

\newcommand{\nphi}{n_{\mathbf{\phi}}}
\newcommand{\mii}{\mathbb{I}}
\newcommand{\miiphi}{\mathbb{I}_{\mathbf{\phi}}}
\newcommand{\miin}{\widehat{\mathbb{I}}_{n}}
\newcommand{\ejxy}{\mathbb{E}_{(\mathbf{X},Y)}}
\newcommand{\emxy}{\mathbb{E}_{\mathbf{X}}\mathbb{E}_{Y}}
\newcommand{\ejxyn}{\mathbb{E}_{(\mathbf{X},Y)_n}}

\newcommand{\ex}{\mathbb{E}_{\mathbf{X}}}
\newcommand{\ey}{\mathbb{E}_{Y}}
\newcommand{\exn}{\mathbb{E}_{\mathbf{X}_n}}
\newcommand{\eyn}{\mathbb{E}_{Y_n}}

\DeclareMathOperator{\PMI}{PMI}
\DeclareMathOperator{\Diff}{Diff}

\newcommand{\STAB}[1]{\begin{tabular}{@{}c@{}}#1\end{tabular}}

\usepackage{subcaption}

\begin{document}

%%%%%%%%% TITLE
\title{
% 1. Reconsider Softmax and Cross-Entropy: \\
% 1. A Pathway Connecting Classification Neural Networks and Mutual Information
% 2. From Mutual Information to Softmax Cross Entropy: \\
% 2. Information Theoretic Perspectives of Classification Neural Networks
% 3. Bridging the Gap Between Mutual Information and Classification Networks
% shorten version of 1:
Rethinking Softmax with Cross-Entropy: \\
Neural Network Classifier as \\ Mutual Information Estimator
}

% \author[1]{\small Zhenyue Qin}
% \author[1,2]{\small Dongwoo Kim}

% \affil[1]{\footnotesize Research School of Computer Science, Australian National University, Australia}
% \affil[2]{\footnotesize Department of Computer Science, Pohang University of Science and Technology, Republic of Korea}
% \affil[ ]{\textit {\{zhenyue.qin,dongwoo.kim\}@anu.edu.au}}

\author[1,*]{Zhenyue Qin}
\author[1,2,*]{Dongwoo Kim}
\author[1]{Tom Gedeon}

\affil[1]{Australian National University, Australia}
\affil[2]{Pohang University of Science and Technology, Republic of Korea}
\affil[ ]{zhenyue.qin@anu.edu.au, dongwookim@postech.ac.kr, tom@cs.anu.edu.au}
\affil[*]{Equal contribution and correspondence}

\maketitle
%\thispagestyle{empty}

%%%%%%%%% ABSTRACT
\begin{abstract}
   Cross-entropy loss with softmax output is a standard choice to train neural network classifiers. While it is reasonable to reduce the cross-entropy between outputs of a neural network and labels, the implication of cross-entropy with softmax on the relation between inputs and labels remains to be better explained. We show that training a neural network with cross-entropy maximises the mutual information between inputs and labels through a variational form of mutual information. Our result provides an alternative view: neural network classifiers are mutual information estimators. The new view leads us to develop an informative class activation map (infoCAM). Given a classification task, infoCAM can locate the most informative features of the input toward a label. When applied to an image classification task, infoCAM performs better than the traditional class activation map in the weakly supervised object localisation task.
\end{abstract}

\section{Introduction}

Neural network classifiers play an important role in contemporary machine learning and computer vision~\cite{lecun2015deep}. Since the emergence of AlexNet~\cite{krizhevsky2012imagenet}, much research has been done to improve the performance of neural network classifiers. To overcome the vanishing gradient in deep networks, the residual connection and various activation functions have been proposed~\cite{he2016deep,nair2010rectified,maas2013rectifier}. To improve generalisation, better regularisation techniques such as dropout have been developed~\cite{srivastava2014dropout}. To reach better local minima, various optimisation techniques have been suggested~\cite{duchi2011adaptive,kingma2014adam}.
Although many architectural choices and optimisation methods have been explored, relatively fewer considerations have been shown on the final layer of the neural network classifier: the cross-entropy loss with the softmax output.

The combination of softmax with cross-entropy is a standard choice to train neural network classifiers. It measures the cross-entropy between the ground truth label $y$ and the output of the neural network $\hat{y}$. The network's parameters are then adjusted to reduce the cross-entropy via back-propagation. While it seems sensible to reduce the cross-entropy between the labels and predicted probabilities, it still remains a question as to what relation the network aims to model between input $x$ and label $y$ via this loss function, \ie, softmax with cross-entropy. 

In this work, for neural network classifiers, we explorer the connection between \emph{cross-entropy with softmax} and \emph{mutual information between inputs and labels}. From a variational form of mutual information, we prove that optimising model parameters using the softmax with cross-entropy is equal to maximising the mutual information between input data and labels when the distribution over labels is uniform. This connection provides an alternative view on neural network classifiers: they are mutual information estimators. We further propose a probability-corrected version of softmax that relaxes the uniform distribution condition. 

This new information-theoretic view of neural network classifiers being mutual information estimators allows us to directly access the most informative regions of input with respect to the labels, given classification tasks. The access to the most informative regions for the labels leads us to develop infoCAM that can locate the most relevant regions for the labels within an image, given an object classification task. Compared to the traditional class activation map, infoCAM exhibits better performance in the weakly supervised object localisation task.

In summary, we outline our contributions as follows: 
\begin{itemize}
    \item The previous view on cross-entropy with softmax only reflects the relationship between the outputs and the labels. We show that with minor modifications to softmax, neural network classifiers then become mutual information estimators. As a result, these mutual information estimators exhibit the information-theoretic relationship between the inputs and the labels. 
    \item We empirically demonstrate that our mutual information estimators can \emph{accurately} evaluate mutual information. We also show mutual information estimators can perform classification more accurately than traditional neural network classifiers. When the dataset is imbalanced, the estimators outperform the state-of-the-art classifier for our example. 
    \item We propose the informative class activation map (infoCAM) which locates the most informative regions for the labels within an image via mutual information. For the weakly supervised object localisation task, we achieve a new state-of-the-art result on Tiny-ImageNet with infoCAM.
\end{itemize}

\section{Preliminaries}
\label{sec:pre}
In this section, we first define the notations used throughout this paper. We then introduce the definition of mutual information and variational forms of mutual information. 

\subsection{Notation}
We let training data consist of $M$ classes and $N$ labelled instances as $\{ (\mathbf{x}_{i}, y_{i}) \}_{i=1}^{N}$, where $y_i \in \mathcal{Y} = \{ 1, ... , M \}$ %\remove{of each datum $(\mathbf{x}_n, y_n)$} 
is a class label of input $\mathbf{x}_i$. We let $n_{\mathbf{\phi}}(\mathbf{x}): \mathcal{X} \rightarrow \mathbb{R}^M$ be a neural network parameterised by $\phi$, where $\mathcal{X}$ is a space of input $\mathbf{x}$. Without additional clarification, we assume $\mathcal{X}$ to be a compact subset of $D$-dimensional Euclidean space. 
We denote by $P_{XY}$ some joint distribution over $\mathcal{X} \times \mathcal{Y}$, with $(\mathbf{X}, Y) \sim P_{XY}$ being a pair of random variables. $P_{X}$ and $P_{Y}$ are the marginal distributions of $\mathbf{X}$ and $Y$, respectively. We remove a subscript from the distribution if it is clear from context.

\subsection{Variational Bounds of Mutual Information}
Mutual information evaluates the mutual dependence between two random variables. 
The mutual information between $\mathbf{X}$ and $Y$ can be expressed as:
\begin{equation}
\label{eq:orig_mi_def}
\mathbb{I}(\mathbf{X}, Y) = \int_{\mathbf{x} \in \mathcal{X}} \bigg[ \sum_{y \in \mathcal{Y}} P(\mathbf{x}, y) \log \big( \frac{P(\mathbf{x}, y)}{P(\mathbf{x}) P(y)} \big) \bigg] d\mathbf{x}. 
\end{equation}

Equivalently, following~\cite{poole2019variational}, we may express the definition of mutual information in \autoref{eq:orig_mi_def} as:
\begin{equation}
\label{eq:cond_mi_def}
\mathbb{I}(\mathbf{X}, Y) = \mathbb{E}_{(\mathbf{X}, Y)} \bigg[ \log \frac{P(y | \mathbf{x})}{P(y)} \bigg],
\end{equation}
where $\mathbb{E}_{(\X, Y)}$ is the abbreviations of $\mathbb{E}_{(\X, Y) \sim P_{XY}}$. 
Computing mutual information directly from the definition is, in general, intractable due to integration. %We therefore derive its variational lower bound to make it tractable. 

%\subsubsection{Barber-Agakov (BA) Representation}
% \label{subsec:ba_rep}
\textbf{Variational form}: Barber and Agakov introduce a commonly used lower bound of mutual information via a variational distribution $Q$~\cite{barber2003algorithm}, derived as:
\begin{align}
\label{eq:ba_def}
\mathbb{I}(\mathbf{X}, Y) &= \mathbb{E}_{(\X, Y)} \bigg[ \log \frac{P(y | \mathbf{x})}{P(y)} \bigg] \notag \\
&= \mathbb{E}_{(\X, Y)} \bigg[ \log \frac{Q(y | \mathbf{x})}{P(y)} \frac{P(y | \mathbf{x})}{Q(y | \mathbf{x})} \bigg] \notag \\
&= \mathbb{E}_{(\X, Y)} \bigg[ \log \frac{Q(y | \mathbf{x})}{P(y)}
\bigg] + 
\underbrace{\mathbb{E}_{(\X, Y)} \bigg[ \log \frac{P(\mathbf{x},y)}{Q(\mathbf{x},y)} \bigg]}_{D_{KL}(P(\mathbf{x}, y) || Q(\mathbf{x}, y))} \notag - 
\underbrace{\mathbb{E}_{(\X)} \bigg[ \log \frac{P(\mathbf{x})}{Q(\mathbf{x})} \bigg]}_{D_{KL}(P(\mathbf{x}) || Q(\mathbf{x}))} \notag \\
&\ge \mathbb{E}_{(\X, Y)} \bigg[ \log \frac{Q(\mathbf{x},y)}{P(\x)P(y)} \bigg]. 
\end{align}
The inequality in \autoref{eq:ba_def} holds since KL divergence maintains non-negativity. This lower bound is tight when variational distribution $Q(\mathbf{x},y)$ converges to joint distribution $P(\mathbf{x},y)$, i.e., $Q(\mathbf{x},y) = P(\mathbf{x},y)$.

The form in \autoref{eq:ba_def} is, however, still hard to compute since it is not easy to make a tractable and flexible variational distribution $Q(\mathbf{x},y)$. Variational distribution $Q(\mathbf{x},y)$ can be considered as a constrained function which has to satisfy the probability axioms. Especially, the constraint is challenging to model with a function estimator such as a neural network. To relax the function constraint, McAllester \etal \cite{mcallester2018formal} further apply reparameterisation and define $Q(\mathbf{x},y)$ in terms of an unconstrained function $f_{\phi}$ parameterised by $\phi$ as:
\begin{equation}
\label{eq:repara_q}
Q(\mathbf{x},y) = \frac{P(\x)P(y)}{E_{y' \sim P_Y}[ \exp(f_{\mathbf{\phi}}(\mathbf{x}, y')) ]} \exp(f_{\mathbf{\phi}}(\mathbf{x}, y)).
\end{equation}
As a consequence, the variational lower bound of mutual information $\mathbb{I}(\mathbf{X}, Y)$ can be rewritten with function $f_{\mathbf{\phi}}$ as:
\begin{align}
\label{eq:mi_f_low_bound}
\mathbb{I}(\X, Y) \ge \mathbb{E}_{(\X, Y)} \bigg[ \log \frac{\exp(f_{\mathbf{\phi}}(\mathbf{x}, y))}{ E_{y'}[ \exp(f_{\mathbf{\phi}}(\mathbf{x}, y')) ]}  \bigg].
\end{align}
Thus, one can estimate mutual information without any constraint on $f$.
%where $f_{\mathbf{\phi}}$ is modelled by any neural network thanks to the removal of the function constraint~\cite{mcallester2018formal}. 
Through the reparameterisation, the MI estimation can be recast as an optimisation problem.

\section{NN Classifiers as MI Estimators}
\label{sec:nn_as_mi_evaluators}
In this section, we prove that a neural network classifier with cross entropy loss and softmax output estimates the mutual information between inputs and labels.
% To the best of the authors' knowledge, neural network classifiers can be divided into two classes: sigmoid-based classifiers and softmax-based ones. The former is for binary classification tasks, while the latter is to tackle situations with more than two classes. Despite seemingly different forms of these two normalisation functions, sigmoid and softmax are mathematically equivalent, given only two classes exist. That is, sigmoid can be considered as a special case of softmax, as shown in \autoref{eq:sig_soft_sig} and \autoref{eq:sig_soft_soft}:
% \begin{align}
%     \label{eq:sig_soft_sig}
%     \sigma_{\text{sig}}(z') = \frac{1}{1 + e^{-z'}},
% \end{align}
% and 
% \begin{align}
%     \label{eq:sig_soft_soft}
%     \sigma_{\text{soft}}(z_0) = \frac{e^{z_0}}{e^{z_0} + e^{z_1}} 
%     = \frac{1}{1 + e^{-(z_0 - z_1)}}. 
% \end{align}
% These two equations are equivalent if $z'$ is $z_0 - z_1$, where $z_0$ and $z_1$ stand for the softmax outputs for classes $0$ and $1$. Therefore, due to the equivalence between sigmoid and softmax, we omit the specific derivations and proofs merely for sigmoid, only focusing on softmax, to reduce clutter and duplication. 

\begin{comment}
\section{Connecting Mutual Information to Softmax}
\label{sec:conn}
In this section, we show the connection between mutual information and the classification neural network.
\end{comment}

To view neural network classifiers as mutual information estimators, we need to discuss two separate cases related to the dataset: whether it is balanced or imbalanced. 

\subsection{Softmax with Balanced Dataset}
Softmax is widely used to map outputs of neural networks into a categorical probabilistic distribution for classification. Given neural network $n(\x):\mathcal{X} \rightarrow \mathbb{R}^{M}$, softmax $\sigma:\mathbb{R}^{M} \rightarrow \mathbb{R}^{M}$ is defined as:
\begin{align}
\sigma(n(\x))_y = \frac{\exp( n(\x)_y )}{\sum_{y'=1}^{M}\exp( n(\x)_{y'})}.
\end{align}
Expected cross-entropy is often employed to train a neural network with softmax output. The expected cross-entropy loss is
\begin{align}
\label{eq:cross_entropy}
L = - \mathbb{E}_{(\X,Y)}[ n(\x)_y - \log({\sum_{y'=1}^{M}\exp( n(\x)_{y'})}) ],
\end{align}
where the expectation is taken over the joint distribution of $X$ and $Y$. Given a training set, one can train the model with an empirical distribution of the joint distribution. We present an interesting connection between cross-entropy with softmax and mutual information in the following theorem. In a bid for conciseness, we only provide proof sketches for \autoref{thm:equality} and \autoref{thm:softmax_im} here. Please refer to the appendix for rigorous proofs. 
\begin{comment}
\begin{thm}
 \label{thm:equality}
 Let $f_\phi(\x,y)$ be $n(\x)_y$. The lower bound of mutual information in \autoref{eq:mi_f_low_bound} can be obtained by minimising the expected cross-entropy with softmax for classification up to constant $\log M$ under the uniform label distribution.
\end{thm}
\end{comment}
\begin{thm}
 \label{thm:equality}
 Let $f_\phi(\x,y)$ be $n(\x)_y$. Infimum of the expected cross-entropy loss with softmax outputs is equivalent to the mutual information between input and output variables up to constant $\log M$ under uniform label distribution. %  maximising \autoref{eq:mi_f_low_bound}, \ie the lower bound of mutual information, under the uniform label distribution. That is, if the dataset is balanced, then training a neural network via minimising cross-entropy with softmax equals enhancing a estimator toward more accurately evaluating the mutual information between data and label. 
\end{thm}

\begin{proof}
Let $f_\phi(\x,y) = n(\x)_y$, then the lower bound is
\begin{align}
 \mathbb{E}_{(\X, Y)} \bigg[ \log \frac{\exp(n(\x)_y)}{ E_{y'}[ \exp(n(\x)_{y'}) ]}  \bigg].
\end{align}
If the distribution of the label is uniform then, it can be rewritten as
\begin{align}
 &\mathbb{E}_{(\X, Y)} \bigg[ \log \frac{\exp(n(\x)_y)}{ 1/M \sum_{y'=1}^{M} \exp(n(\x)_{y'}) }  \bigg] \notag \\
 &= \mathbb{E}_{(\X, Y)} \bigg[ \log \frac{\exp(n(\x)_y)}{ \sum_{y'=1}^{M} \exp(n(\x)_{y'}) } \bigg] + \log M, \label{eq:softmax_mi}
\end{align}
which is equivalent to the negative expected cross-entropy loss (\ref{eq:cross_entropy}) up to constant $\log M$. Hence, the infimum of the expected cross entropy is equal to the mutual information between input and output variables since the supremum of r.h.s in \autoref{eq:mi_f_low_bound} is the mutual information. %Hence, by minimising the cross-entropy, we can obtain the lower bound of mutual information.
\end{proof}
Note that the constant does not change the gradient of the objective. Consequently, the solutions of both the mutual information maximisation and the softmax cross-entropy minimisation optimisation problems are the same. 

\subsection{Softmax with Imbalanced Dataset}
% Please check (online) the definition of Imbalanced and Unbalanced - I think you mean Unbalanced
The uniform label distribution assumption in \autoref{thm:equality} is restrictive since we cannot access the true label distribution, often assumed to be non-uniform. To relax the restriction, we propose a probability-corrected softmax (PC-softmax):
\begin{align}
\label{eq:prob_cor_softmax}
\sigma_p(n(\x))_y = \frac{\exp( n(\x)_y )}{\sum_{y'=1}^{M}P(y')\exp( n(\x)_{y'})},
\end{align}
where $P(y')$ is a distribution over label $y'$. 
% For the sake completeness, even though sigmoid is a special case of softmax, we include the pc-sigmoid as well: 
% \begin{align}
% \label{eq:pc_sigmoid}
% \sigma_s(n(\x))_y = \frac{1 / P(y_1)}{1 + (P(y_0) / P(y_1)) \exp( n(\x)_{y_1})},
% \end{align}
% where $y_1$ and $y_0$ stand for belonging to and not belonging to the class, respectively. 
In experiments, we optimise the revised softmax with empirical distribution on ${P}(y')$ estimated from the training set. We show the equivalence between optimising the classifier and maximising mutual information with the new softmax below.

\begin{comment}
\begin{thm}
\label{thm:softmax_im}
The mutual information between two random variable $X$ and $Y$ can be obtained via the infimum of cross-entropy with PC-softmax in \autoref{eq:prob_cor_softmax} under a mild condition on $n$.
\end{thm}
\end{comment}

\begin{thm}
\label{thm:softmax_im}
The mutual information between two random variables $X$ and $Y$ can be obtained via the infimum of cross-entropy with PC-softmax in \autoref{eq:prob_cor_softmax}, using a neural network. Such an evaluation is strongly consistent. % Not sure what this last sentence means
\end{thm}
% See \autoref{sec:proofs} for the proof of \autoref{thm:softmax_im}.
See the proofs in the appendix for the proof of \autoref{thm:softmax_im}.
% \begin{proof}
%It has been shown that the equality in \autoref{eq:ba_def} holds iff r.h.s of \autoref{eq:mi_f_low_bound} is in its supremum~\cite{mcallester2018formal}. 
%First, it can be easily shown that we can relax the uniform assumption with PC-softmax. We then show that the class of functions modelled by $n:\mathcal{X} \rightarrow \mathbb{R}^{M}$ is the same as those of $f:\mathcal{X}\times\mathcal{Y} \rightarrow \mathbb{R}$.
% $Y$ is a categorical variable. Hence, unconstrained function $f$ can be decomposed into a set of functions indexed by $y$, i.e., $f = \{f_{y} \}_{y=1}^{M}$. Under a mild condition, $n(x)_y$ can approximate any continuous function by the universal approximation theorem~\cite{hornik1989multilayer}. We conclude the sketch proof by letting $n(x)_y$ be $f_{y}$ so that $n(x)_y$ is the mutual information evaluator.
% \end{proof}

Mutual information is often used in generative models to find the maximally informative representation of an observation~\cite{hjelm2018learning,zhao2017infovae}, whereas its implication in classification has been unclear so far. The results of this section imply that the neural network classifier with softmax optimises its weights to maximise the mutual information between inputs and labels under the uniform label assumption. 
% We further study an application of this implication in \autoref{sec:cam} to tackle the weakly supervised object localisation task.
%In \autoref{sec:mi_class_exp}, we show the empirical difference between softmax and PC-softmax via synthetic and real world datasets on the mutual information estimation and classification tasks.

\begin{table}[t!]
    \centering
    \begin{tabular}{c r r r}
        \toprule
        $y$ & $\mu$ & \# samples & $p(y)$\\
        \midrule 
        0 & $\mathbf{0}$ & 6,000 & 0.07 \\
        1 & $+\mathbf{2}$ & 12,000 & 0.13 \\
        2 & $-\mathbf{2}$ & 18,000 & 0.20 \\
        3 & $+\mathbf{4}$ & 24,000 & 0.27 \\
        4 & $-\mathbf{4}$ & 30,000 & 0.33 \\
        \bottomrule
    \end{tabular}
    \vspace{1em}
    \caption{Synthetic dataset description. $\mu$ is a mean vector for each Gaussian distribution. \# samples denotes the number (resp. prior distribution) of samples with the non-uniform prior assumption. For the test with the uniform prior assumption, we use 12,000 samples from each distribution.}
    \label{table:synthetic_dataset_spec}
\end{table}

\begin{table}[t!]
\begin{subtable}[t]{\linewidth}
\centering
\begin{tabular}{r r r r r}
\toprule
\multirow{2}{*}{Dimension}&
\multirow{2}{*}{Accuracy(\%)}& 
\multicolumn{3}{c}{Mutual information}\\
\cmidrule(lr{1em}){3-5}
& & MC & MINE & softmax \\
\midrule
1 & 74 & 1.03 & 1.00 & 0.99\\
2 & 85 & 1.30 & 1.22 & 1.28\\
5 & 94 & 1.54 & 1.46 & 1.48\\
10 & 98 & 1.60 & 1.54 & 1.54\\
\bottomrule
\end{tabular}
\caption{Results with balanced datasets.}
\label{table:syn_balanced}
\end{subtable}

\vspace{1em}

\begin{subtable}[t]{\linewidth}
\centering
\begin{tabular}{r r r r r r r}
\toprule
\multirow{2}{*}{Dimension}&
\multicolumn{2}{c}{Accuracy(\%)}&
\multicolumn{4}{c}{Mutual information}\\
\cmidrule(lr{1em}){2-3}
\cmidrule(lr{1em}){4-7}
& softmax & PC-softmax & MC & MINE & softmax & PC-softmax \\
\midrule
1 & 79 & 79 & 1.02 & 0.99 & 1.11 & 0.96 \\
2 & 87 & 88 & 1.23 & 1.17 & 1.31 & 1.20 \\
5 & 93 & 95 & 1.44 & 1.27 & 1.41 & 1.31 \\
10 & 95 & 96 & 1.48 & 1.22 & 1.36 & 1.34 \\
\bottomrule
\end{tabular}
\caption{Results with unbalanced datasets.} % Acc. stands for the classification accuracy with Softmax and PC-Softmax, respectively. } % Acc. is also defined in the caption below
\label{table:syn_imbalanced}
\end{subtable}

\caption{Mutual information estimation results with softmax-based classification neural networks. MC represents the estimated mutual information via Monte Carlo methods.}
\label{table:cam_info_cam}
\end{table}

\section{Impact of PC-softmax on Classification}
\label{sec:mi_class_exp}
%In the previous section, we show that classification neural networks can be utilised to measure the mutual information (MI) between continuous and discrete distributions. 
In this section, we measure the empirical performance of PC-softmax as mutual information (MI) and the influence of PC-softmax on the classification task. Since it is impossible to obtain correct MI from real-world datasets, we first construct synthetic data with known properties to measure the MI estimation performance, and then we use two real-world datasets to measure the impact of PC-softmax on classification tasks.
%We display both softmax and the new probability-corrected (PC) softmax outlined in \autoref{eq:prob_cor_softmax} can approximate MI even under significantly imbalanced dataset, despite the correctness of the corrected approach.

\subsection{Mutual information estimation task}

To construct a synthetic dataset with a pair of continuous and discrete variables, we employ a Gaussian mixture model:
\begin{align}
    P(x) &= \sum_{y=1}^{M} P(y) \mathcal{N}(\mathbf{x} | \mathbf{\mu}_y, \mathbf{\Sigma}_y) \notag\\
    P(x | y) &= \mathcal{N}(\mathbf{x} | \mathbf{\mu}_y, \mathbf{\Sigma}_y), \notag
\end{align}
where $P(y)$ is a prior distribution over the labels. To form a classification task, we use $x$ as an input variable, and $y$ as a label.

For the experiments, we use five mixtures of isotropic Gaussian, each of which has a unit diagonal covariance matrix with different means. We set the parameters of the mixtures to make them overlap in significant proportions of their distributions.

% Separate discussion on balance
We generate two sets of datasets: one with uniform prior and the other with non-uniform prior distribution over labels, $p(y)$. For the uniform prior, we sample 12,000 data points from each Gaussian, and for the non-uniform prior, we sample unequal number of data points from each Gaussian. In addition, we vary the dimension of Gaussian distribution from 1 to 10. The detailed statistics for the Gaussian parameters and the number of samples are available in \autoref{table:synthetic_dataset_spec}. To train classification models, we divide the dataset into training, validation and test sets. We use the validation set to find the best parameter configuration of the classifier.

We aim to compare the difference of true and softmax-based estimated mutual information $\mathbb{I}(\mathbf{X}, Y)$. 
The mutual information is, however, intractable. We thus approximate it via Monte Carlo (MC) methods using the true probability density function, expressed as:
\begin{equation}
    \label{eq:mc_mi}
    \mathbb{I}(\X, Y) \approx \frac{1}{N} \sum_{i=1}^{N} \log \left( \frac{P(\mathbf{x}_i | y_i)}{P(\mathbf{x}_i)} \right), 
\end{equation}
where $(\mathbf{x}_i, y_i)$ forms a paired sample. \autoref{eq:mc_mi} attains equality as $N$ approaches infinity. 

%To estimate mutual information via classification, we train classification models with two versions of softmax.

We use four layers of a feed-forward neural network with the ReLU as an activation for internal layers and softmax as an output layer\footnote{All model details used in this paper are available in the supplementary material.}. We train the model with softmax on balanced dataset and with PC-softmax on unbalanced dataset. We compare the experimental results against mutual information neural estimator (MINE) proposed in \cite{belghazi2018mutual}. Note that MINE requires having a pair of input and label variables as an input of an estimator network, the classification-based MI-estimator seems more straightforward for measuring mutual information between inputs and labels of classification tasks.

\autoref{table:syn_balanced} summarises the experimental results with the balanced dataset. With the balanced dataset, there is no difference between softmax and PC-softmax. Note that the MC estimator has access to explicit model parameters for estimating mutual information, whereas the softmax estimator measures mutual information based on the model outputs without accessing the true distribution. We could not find a significant difference between MC and the softmax estimator. Additionally, we report the accuracy of the trained model on the classification task.

\autoref{table:syn_imbalanced} summarises the experimental results with the unbalanced dataset. The results show that the PC-softmax slightly under-estimates mutual information when compared with the other two approaches. It is worth noting that the classification accuracy of PC-softmax consistently outperforms the original softmax. 
The results show that the MINE slightly under-estimate the MI as the input dimension increases.
%However, MI evaluator based on PC-softmax does not require the marginal distributions of inputs and labels, which may potentially reduce the complexity of model architecture.
%is less intensive since ours does not require taking both joint and marginal %word missing here: joint and marginal <something> of data ...
%of data and labels as inputs. 
% although it is not statistically significant.

% there will be a classification result on the real world dataset mnist and cub at the end of this section.

\subsection{Classification task}
%In this section, we demonstrate that maximising mutual information can result in classifiers with higher classification accuracy. 

We test the classification performance of softmax and PC-softmax with two real-world datasets: MNIST~\cite{lecun2010mnist} and CUB-200-2011~\cite{wah2011caltech}. 
%We utilise the former due to its simplicity and classic for evaluating a classifier. However, MNIST can be overly simple for assessing modern classifiers. Thus, we apply a more challenging dataset that is still popular for testing the current classifiers.

%To be more specific with the dataset settings, w
We construct balanced and unbalanced versions of the MNIST dataset.
For the balanced-MNIST, we use a subset of the original dataset. For the unbalanced-MNIST, we randomly subsample one tenth of instances for digits 0, 2, 4, 6 and 8 from the balanced-MNIST. 
With CUB-200-2011, we follow the same training and validation splits as in~\cite{cui2018large}. As a result of such splitting, the training set is approximately balanced, where out of the total 200 classes, 196 of them contain 30 instances and the remaining 6 classes include 29 instances. To construct an unbalanced dataset, similar to MNIST, we randomly drop one half of the instances from one half of the bird classes.

We adopt a simple convolutional neural network as a classifier for MNIST. The model contains two convolutional layers with max pooling layer and the ReLU activation, followed by two fully connected layers with the final softmax. For CUB-200-2011, we apply the same architecture as Inception-V3~\cite{cui2018large}.%, which demonstrates the state-of-the-art classification performance in CUB-200-2011 after being fine-tuned~\cite{cui2018large}. 
We measure both the micro accuracy and the average per-class accuracy of the two softmax versions on both datasets. The average per-class accuracy alleviates the dominance of the majority classes in unbalanced datasets. The classification results are shown in \autoref{table:emp_soft}. PC-softmax is significantly more accurate than softmax on unbalanced datasets in terms of the average per-class accuracy. %supported by the Mann-Whitley statistical test results
%, while the results are similar for both softmax versions on standard classification accuracy. 

%however, the difference is not statistically significant. %There is no significant difference in performance between two softmax.

\begin{table}[t!]
\begin{subtable}[t]{\linewidth}
\centering
\begin{tabular}{c r r r r r}
\toprule
\multirow{2}{*}{Dataset} & \multicolumn{2}{c}{MNIST} & \multicolumn{2}{c}{CUB-200-2011} \\
&
Balanced& 
Unbalanced& 
Balanced& 
Unbalanced\\
\midrule
softmax & 97.95 & 96.81 & 89.23 & 89.21  \\
PC-softmax & 97.91 & 96.86 & 89.18 & \textbf{89.73}* \\
%p-value & 0.28& 0.01& 0.40& 0.01 \\
\bottomrule
\end{tabular}
\caption{Classification accuracy (\%).}

\end{subtable}

\vspace{1em}

\begin{subtable}[t]{\linewidth}
\centering
\begin{tabular}{c r r r r r}
\toprule
\multirow{2}{*}{Dataset} & \multicolumn{2}{c}{MNIST} & \multicolumn{2}{c}{CUB-200-2011} \\
& 
Balanced& 
Unbalanced& 
Balanaced& 
Unbalanced\\
\midrule
softmax & 97.95 & 95.05 & 89.21 & 84.63 \\
PC-softmax & 97.91 & \textbf{96.30} & 89.16 & \textbf{87.69} \\
%p-value & 0.28& 0.01& 0.83& 0.01 \\

\bottomrule
\end{tabular}
\caption{Average per-class accuracy (\%).}
\end{subtable}

\vspace{1em}
\caption{Classification accuracy of using softmax and PC-softmax. Numbers of instances for different labels are the same for a balanced dataset and are significantly distinct for an unbalanced dataset. Bold values denote p-values less than 0.05 with the Mann-Whitney U test\protect\footnotemark.}
\label{table:emp_soft}
\end{table}
\footnotetext{Accuracy with * is higher than the current state-of-the-art~\cite{cui2018large}.}
% \section{Weakly Supervised Object Localisation}
\section{Informative Class Activation Maps: \\
Estimating Mutual Information Between Regions and Labels
}
\label{sec:cam}
In this section, we show that viewing neural network classifiers as mutual information estimators contributes to a more interpretable neural network classifier, via identifying regions of an image that contain high mutual information with a label. %As such, humans can better understand the motivations behind a classification result from a neural network classifier. 
There exist previous work exhibiting how to identify regions of an image corresponding to particular labels, known as class activation maps (CAM). We further formalise CAMs to be related to information theory. Furthermore, with the new view of neural network classifiers as mutual information evaluators, we are able to depict the quantitative relationship between the information of the entire image and its local regions about a label. We call our new CAM Informative Class Activation Map (infoCAM), since it is based on information theory. Moreover, infoCAM can also improve the performance of the weakly supervised object localisation (WSOL) task than the original CAM. 

To explain infoCAM, we first introduce the concept and definition of the class activation map. We then show how to apply it to the weakly supervised object localisation (WSOL) task. 

\begin{figure*}[t!]
    \centering
    \includegraphics[width=\linewidth]{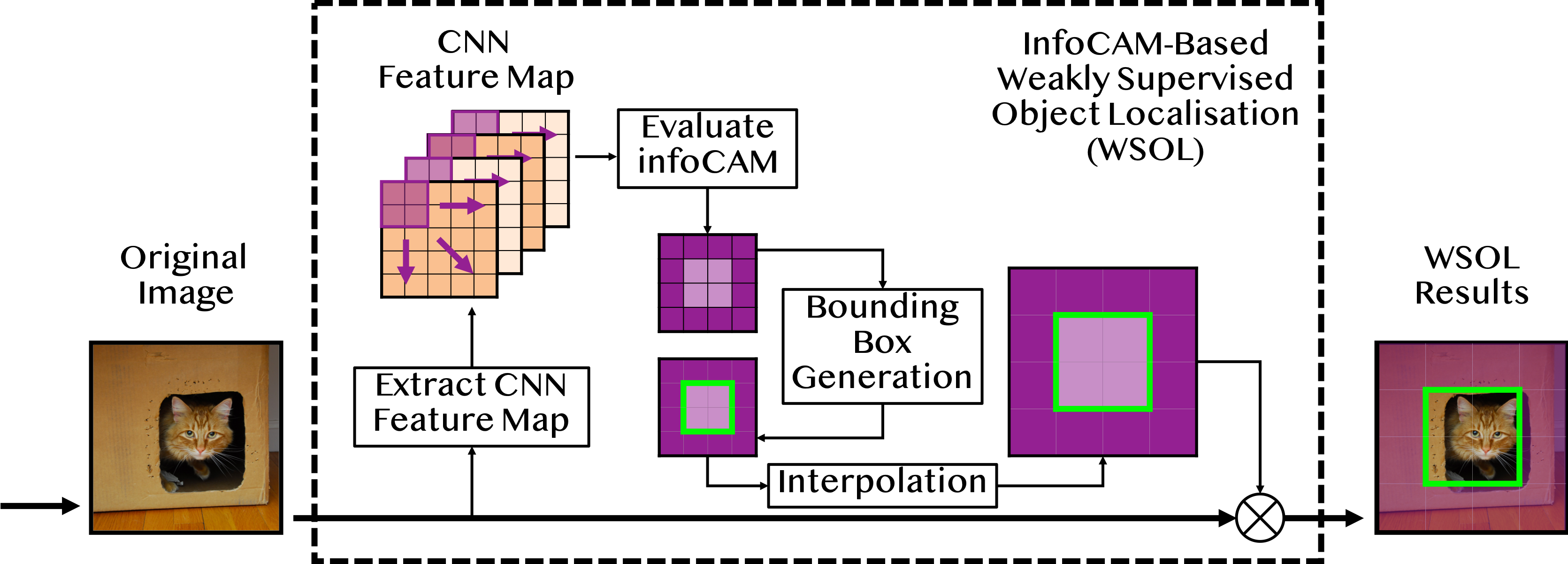}
    \caption{A visualisation of the infoCAM procedure for the WSOL task. The task aims to draw a bounding box for the target object in the original image. The procedure includes: 1) feed input image into a CNN to extract its feature maps, 2) evaluate PMI difference between the true and the other labels of input image for each region within the feature maps, 3) generate the bounding box by keeping the regions exceeding certain infoCAM values and find the largest connected region and 4) interpolate and map the bounding box to the original image.}
    \label{fig:infoCAM-Illustration}
\end{figure*}

\subsection{CAM: Class Activation Map}
% Class activation mapping aims to identify which regions within an input image a convolutional neural network (CNN) concentrates to decide the label of the input image. 
% For example, given a dog image, the class activation map (CAM) identifies that a CNN focuses on the areas where dogs existing. 
% %Well-trained CNNs are often assumed to have a good localisation ability. 
% Localisation task that aims to identify the location of the object in a scene by only accessing to image-level class labels without exact location annotations is known as WSOL~\cite{li2019progressive}. 
% In brief, the procedure of WSOL with CNN-based classifiers involves extracting the image regions whose feature maps displaying high activation scores corresponding to a particular class. 
% Informally, such feature maps weighted by the activation scores of a class are class activation maps. 

Contemporary classification CNNs such as AlexNet~\cite{krizhevsky2012imagenet} and Inception~\cite{szegedy2015going} consist of stacks of convolutional layers interleaved with pooling layers for extracting visual features. 
These convolutional layers result in feature maps. A feature map is a collection of 2-dimensional grids. The size of the feature map depends on the structure of convolution and pooling layers. Generally the feature map is smaller than the original image. The number of feature maps corresponds to the number of convolutional filters.
The feature maps from the final convolutional layer are usually averaged, flattened and fed into the fully-connected layer for classification~\cite{lin2014network}.
Given $K$ feature maps $g_1, .. , g_K$, the fully-connected layer consists of weight matrix $W \in \mathbb{R}^{M \times K}$, where $w_k^y$ represents the scalar weight corresponding to class $y$ for feature $k$.
We use $g_k(a,b)$ to denote a value of 2-dimensional spatial point $(a,b)$ with feature $k$ in map $g_k$.
In~\cite{choe2019attention}, the authors propose a way to interpret the importance of each point in feature maps. The importance of spatial point $(a,b)$ for class $y$ is defined as a weighted sum over features:
\begin{align}
\label{eq:cam_def}
    M_{y}(a, b) = \sum_{k} w_k^{y} g_{k} (a, b).
\end{align}
 We redefine $M_{y}(a, b)$ as an intensity of the point $(a, b)$.
The collection of these intensity values over all grid points forms a class activation map (CAM). CAM highlights the most relevant region in the feature space for classifying $y$.
The input going to the softmax layer corresponding to the class label $y$ is: 
\begin{align}
    \sum_{a,b} M_{y}(a, b) = n(\x)_y.
\end{align}
Intuitively, weight $w_k^y$ indicates the overall importance of the $k$th feature  to class $y$, and intensity $M_{y}(a, b)$ implies the importance of the feature map at spatial location $(a, b)$ leading to the classification of image $\mathbf{x}$ to $y$.

%The feature maps contain relative locations of extracted features in the original image.
The aim of WSOL is to identify the region containing the target object in an image given a label, without any pixel-level supervision.
Previous approaches tackle the WSOL task by creating a bounding box from the CAM~\cite{choe2019attention}. Such a CAM contains all important locations that exceed a certain intensity threshold. The box is then upsampled to match the size of the original image.

\subsection{InfoCAM: Informative Class Activation Map}
In section \ref{sec:nn_as_mi_evaluators}, we show that softmax classifier carries an explicit implication between inputs and labels in terms of information theory. We extend the notion of mutual information from being a pair of an input image and a label to regions of the input image and labels to capture the regions that have high mutual information with labels.

To simplify the discussion, we assume here that there is only one feature map, \ie $K=1$. However, the following results can be easily applied to the general cases where $K>1$ without loss of generality.
We introduce a region $R$ containing a subset of grid points in feature map $g$. %With a slight abuse of notation, we let $g$ be a set of grid points, $|g|$ be the total number of points in the map, and $|R|$ be the number of points in region $R$.

Mutual information is an expectation of the point-wise mutual information (PMI) between two variables, \ie $\mathbb{I}(\X,Y) = \mathbb{E}[\text{PMI}(\x,y)]$. Given two instances of variables, we can estimate their PMI via \autoref{eq:softmax_mi}, \ie
\begin{align}
\PMI(\x, y) = n(\x)_y - \log\sum_{y'=1}^{M}\exp(n(\x)_{y'}) + \log M. \notag
\end{align}
The PMI is close to $\log M$ if $y$ is the maximum argument in log-sum-exp. To find a region which is the most beneficial to the classification, we compute the difference between PMI with true label and the average of the other labels and decompose it into a point-wise summation as
\begin{align}
\Diff(\PMI(\x)) = \PMI(\x, y^*) - \frac{1}{M-1}\sum_{y' \neq y^*}\PMI(\x, y') \notag\\
= \sum_{(a,b)\in g} w^{y*} g(a,b) - \frac{1}{M-1}\sum_{y' \neq y^*} w^{y'} g(a,b). \notag
\end{align}
The point-wise decomposition suggests that we can compute the PMI differences with respect to a certain region. Based on this observation, we propose a new CAM, named informative CAM or infoCAM, with the new intensity function $M_{y}^{\Diff}(R)$ between region $R$ and label $y$ defined as follows:
\begin{align}
M_{y}^{\Diff}(R) = \sum_{(a,b)\in R} w^yg(a,b) - \frac{1}{M-1}\sum_{y' \neq y}w^{y'}g(a,b).
\label{eq:info_cam_correct}
\end{align}

The infoCAM highlights the region which decides the classification boundary against the other labels. 
Moreover, we further simplify \autoref{eq:info_cam_correct} to be the difference between PMI with the true and the most-unlikely labels according to the classifier's outputs, denoting as infoCAM+, with the new intensity: 
\begin{align}
M_{y}^{\Diff^+}(R) = \sum_{(a,b)\in R} w^y g(a,b) - w^{y'}g(a,b), 
\label{eq:info_cam_correct_min}
\end{align}
where $y' = \underset{m}{\arg \min} \sum_{(a,b)\in R} w^{m}g(a,b)$. 

The complete procedure of WSOL with infoCAM is visually illustrated in \autoref{fig:infoCAM-Illustration}. We first feed an input image into a CNN to extract its feature maps. Then instead of computing the CAM of the feature map, we compute infoCAM of varying regions from the input image and the class label. Afterwards, we generate the bounding box for the object by preserving regions surpassing a certain intensity level. Then, we generate the bounding box that covers the largest connected remaining regions~\cite{zhou2016learning}. Finally, we interpolate the generated bounding box to the original image size and merge the two. 
\section{Object Localisation with InfoCAM}
\label{sec:exp}
In this section, we demonstrate experimental results with infoCAM for WSOL. We first describe the experimental settings and then present the results.

\begin{table}[t!]
\centering

\begin{tabular}{c l r r r r}
\toprule
&  &  \multicolumn{2}{c}{CUB-200-2011} & \multicolumn{2}{c}{Tiny-ImageNet}\\
&  & 
\makecell{GT \\ Loc. (\%)} & \makecell{Top-1 \\ Loc. (\%)} & \makecell{GT \\ Loc. (\%)} & \makecell{Top-1 \\ Loc. (\%)}\\ 
\midrule
\multirow{6}{*}{\STAB{\rotatebox[origin=c]{90}{VGG}}} & CAM& 42.49 & 31.38 & 53.49 & 33.48 \\
& CAM (ADL) & 71.59 & 53.01 &  52.75 & 32.26 \\
& {infoCAM} & {52.96} & {39.79} & {55.50} & {34.27}\\
& {infoCAM} (ADL) & {73.35} & {53.80} & {53.95} & {33.05} \\ 
& {infoCAM+} & {59.43} & {44.40} &  {55.25} & {34.27}\\
& {infoCAM+} (ADL) & \textbf{75.89} & {54.35} &  {53.91} & {32.94} \\  \midrule
\multirow{6}{*}{\STAB{\rotatebox[origin=c]{90}{ResNet}}} & CAM& 61.66 & 50.84 &  54.56 & 40.55  \\
& CAM (ADL) & 57.83 & 46.56 &  52.66 & 36.88  \\
& {infoCAM} &  {64.78} & {53.22} & \textbf{57.79} & \textbf{43.34} \\
& {infoCAM} (ADL) & {67.75} & {54.71} & {54.18} & {37.79} \\
& {infoCAM+} &  {68.99} & \textbf{55.83} &  {57.71} & {43.07} \\
& {infoCAM+} (ADL) & {69.63} & {55.20} &  {53.70} & {37.71} \\
\bottomrule
\end{tabular}

\vspace{1em}
\caption{Localisation results of CAM and infoCAM on CUB-2011-200 and Tiny-ImageNet. InfoCAM outperforms CAM on localisation of objects with the same model architecture. Bold values represent the highest accuracy for a certain metric. }
\label{table:cam_info_cam}
\end{table}

\subsection{Experimental settings}
We evaluate WSOL performance on CUB-200-2011~\cite{wah2011caltech} and Tiny-ImageNet~\cite{tiny_imagenet}. CUB-200-2011 consists of 200 bird specifies, including 5,994 training and 5,794 validation images. Each bird class contains roughly the same number of instances, thus the dataset is approximately balanced. Since the dataset only depicts birds, not including other kinds of objects, variations due to class difference are subtle~\cite{dubey2018pairwise}. Therefore, CNN-based classifiers tend to concentrate on the most discriminative areas within an image while disregarding other regions that are similar among all the birds~\cite{wang2019camdrop}. Such nuance-only detection can lead to localisation accuracy degradation~\cite{choe2019attention}. 

Tiny-ImageNet is a reduced version of ImageNet in terms of both class number,  number of instances per class and image resolution. It includes 200 classes, and each consists of 500 training and 50 validation images, and is balanced. Unlike CUB-200-2011 comprising only birds, Tiny-ImageNet contains a wide range of objects from animals to daily supplies. Compared with the full ImageNet, training classifiers on Tiny-ImageNet is faster due to image resolution reduction and quantity shrinkage, yet classification becomes more challenging~\cite{odena2017conditional}.

\begin{figure}[t!]
    \centering
    \includegraphics[width=0.5\linewidth]{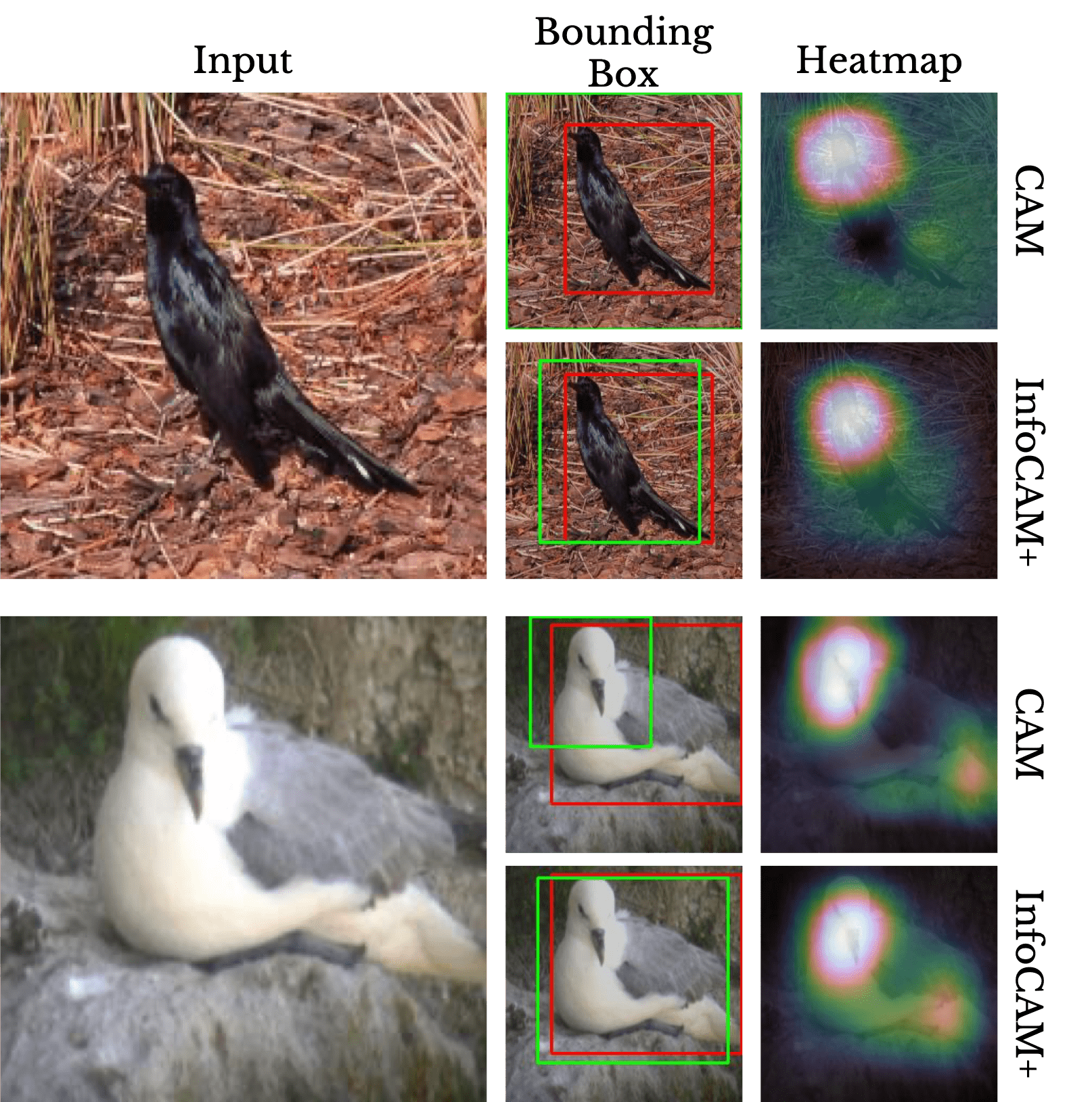}
    \caption{Visualisation of comparison between CAM and infoCAM+. Red and green boxes represent the ground truth and prediction, respectively. Brighter regions represent higher CAM or infoCAM+ values.}
    \label{fig:reg-sig}
\end{figure}

To perform an evaluation on localisation, we first need to generate a bounding box for the object within an image. We generate a bounding box in the same way as in~\cite{zhou2016learning}. Specifically, after evaluating infoCAM within each region of an image, we only retain the regions whose infoCAM values are more than 20\% of the maximum infoCAM and abandon all the other regions. Then, we draw the smallest bounding box that covers the largest connected component. 

We follow the same evaluation metrics in~\cite{choe2019attention} to evaluate localisation performance with two accuracy measures: 1) localisation accuracy with known ground truth class (GT Loc.), and 2) top-1 localisation accuracy (Top-1 Loc.). GT Loc. draws the bounding box from the ground truth of image labels, whereas Top-1 Loc. draws the bounding box from the predicted most likely image label and also requires correct classification. The localisation of an image is judged to be correct when the intersection over union of the estimated bounding box and the ground-truth bounding box is greater than 50\%.

\begin{figure*}[t!]
    \centering
    \includegraphics[width=1.0\linewidth]{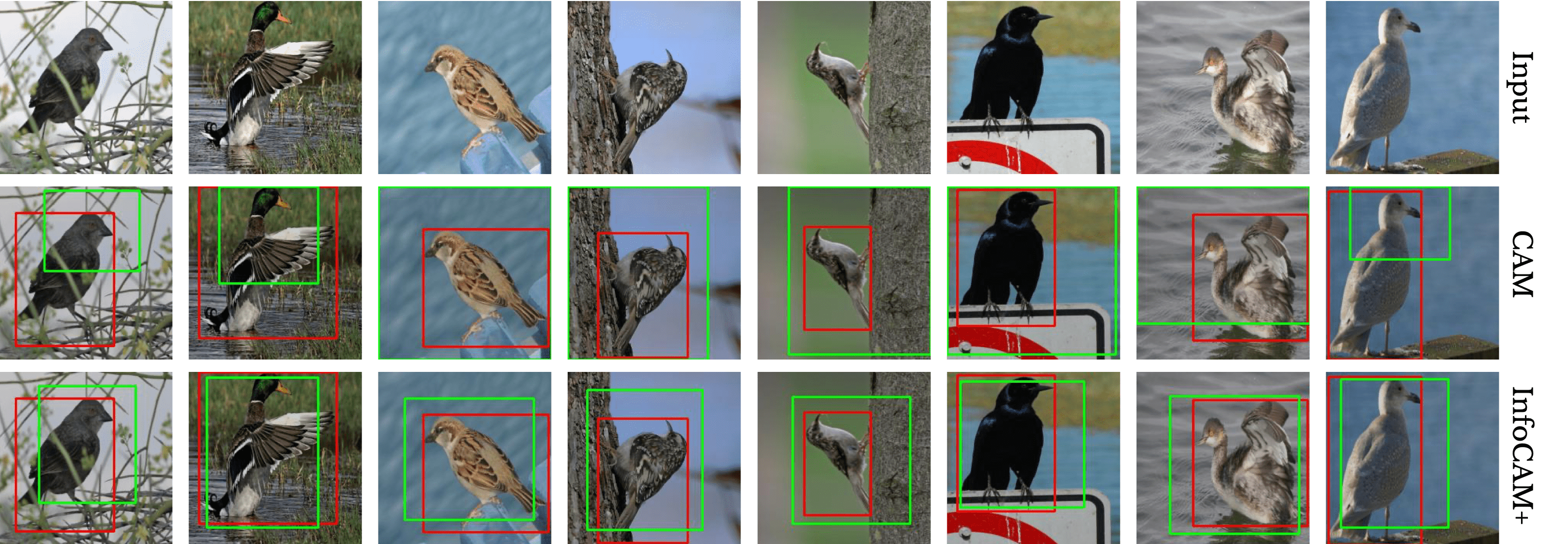}
    \caption{Visualisation of localisation with ResNet50 without using ADL on CUB-200-2011. Images in the second and the third row correspond to CAM and infoCAM+, respectively. Estimated (green) and ground-truth (red) bounding boxes are shown separately.}
    \label{fig:cam-illustration-butterfly}
\end{figure*}

We adopt the same network architectures and hyper-parameters as in~\cite{choe2019attention}, which shows the current state-of-the-art performance. Specifically, the network backbone is ResNet50~\cite{he2016deep} and a variation of VGG16~\cite{szegedy2015going}, in which the fully connected layers are replaced with global average pooling (GAP) layers to reduce the number of parameters. The traditional softmax is used as the final layer since both datasets are well balanced.
InfoCAM requires the region parameter $R$. We apply a square region for the region parameter $R$. The size of the region $R$ is set as $5$ and $4$ for VGG and ResNet in CUB-200-2011, respectively, and $3$ for both VGG and ResNet in Tiny-ImageNet.

%, first introduced in~\cite{lin2014network}
These models are tested with the Attention-based Dropout Layer (ADL) to tackle the localisation degradation problem~\cite{choe2019attention}. ADL is designed to randomly abandon some of the most discriminative image regions during training to ensure CNN-based classifiers cover the entire object. The ADL-based approaches demonstrate state-of-the-art performance in CUB-200-2011~\cite{choe2019attention} and Tiny-ImageNet~\cite{choe2018improved} for the WSOL task and are computationally efficient. We test ADL with infoCAMs to enhance WSOL capability.

To prevent overfitting in the test dataset, we evenly split the original validation images to two data piles, one still used for validation during training and the other acting as the final test dataset. We pick the trained model from the epoch that demonstrates the highest top-1 classification accuracy in the validation dataset and report the experimental results with the test dataset. All experiments are run on two Nvidia 2080-Ti GPUs, with the PyTorch deep learning framework~\cite{paszke2017automatic}. 

\subsection{Experimental Results}

Table~\ref{table:cam_info_cam} shows the localisation results on CUB-200-2011 and Tiny-ImageNet. The results demonstrate that infoCAM can consistently improve accuracy over the original CAM for WSOL under a wide range of networks and datasets. Both infoCAM and infoCAM+ perform comparably to each other. ADL improves the performance of both models with CUB-200-2011 datasets, but it reduces the performance with Tiny-ImageNet. We conjecture that dropping any part of a Tiny-ImageNet image with ADL significantly influences classification since the images are relatively small.

\autoref{fig:reg-sig} highlights the difference between CAM and infoCAM. The figure suggests that infoCAM gives relatively high intensity on the object to compare with that of CAM, which only focuses on the head part of the bird.
Figure~\ref{fig:cam-illustration-butterfly} in the Appendix presents additional examples of visualisation for comparing localisation performance of CAM to infoCAM, both without the assistance of ADL\footnote{Please refer to the supplementary material for more Tiny-ImageNet visualisation results.}. From these visualisations, we notice that 
the bounding boxes generated from infoCAM are formed closer to the objects than the original CAM. That is, infoCAM tends to precisely cover the areas where objects exist, with almost no extraneous or lacking areas. For example, CAM highlights the bird heads in CUB-200-2011, whereas infoCAM also covers the bird bodies. 

\textbf{Ablation Study}: InfoCAM differs from CAM in two ways: 1) the new intensity function and 2) region-based intensity smoothing with parameter $R$. We conduct an ablation study to investigate which feature(s) help to  localise objects.
The results suggest that both components are indispensable to improve the performance of the localisation. For the detailed results, please refer to the ablation study table in the Appendix.

\subsection{Localisation of multiple objects with InfoCAM}
\begin{figure}[ht]
\centering
\begin{subfigure}{.3\textwidth}
  \centering
  % include first image
  \includegraphics[width=.99\linewidth]{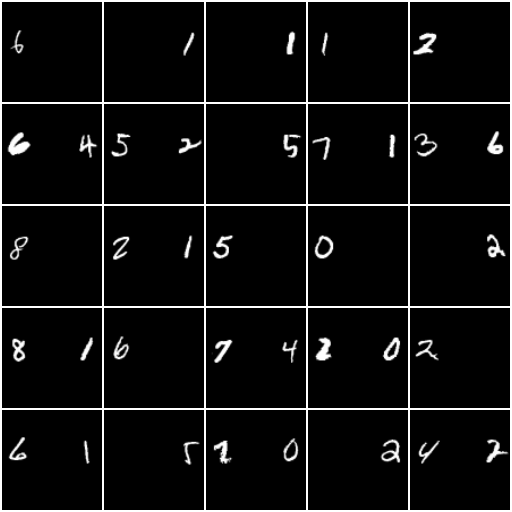}  
  \caption{Original input images. }
  \label{fig:sub-first}
\end{subfigure}
\begin{subfigure}{.3\textwidth}
  \centering
  % include second image
  \includegraphics[width=.99\linewidth]{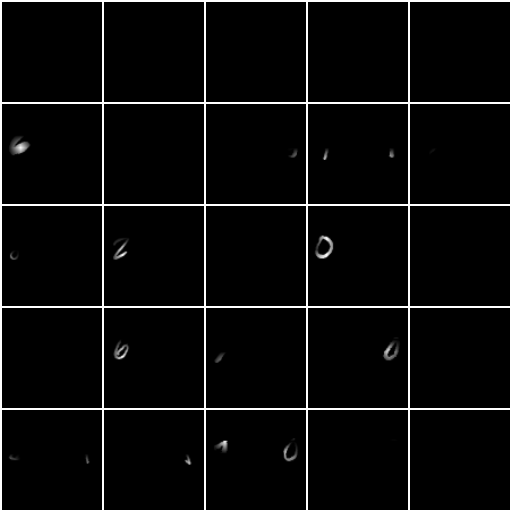}  
  \caption{CAM localisation. }
  \label{fig:sub-second}
\end{subfigure}
\begin{subfigure}{.3\textwidth}
  \centering
  % include second image
  \includegraphics[width=.99\linewidth]{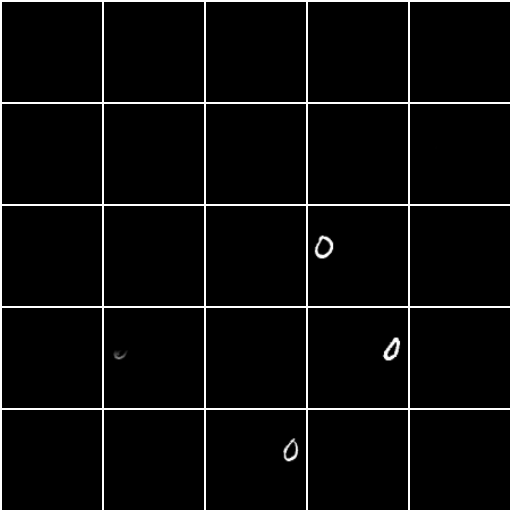}  
  \caption{InfoCAM localisation. }
  \label{fig:sub-second}
\end{subfigure}
\caption{Visualisation of comparison between CAM and infoCAM for the multi-MNIST dataset. Each image has one or two digits in the left and/or right. We aim to extract digit 0 in each image. }
\label{fig:multi_mnist}
\end{figure}
\begin{comment}
\begin{figure}[t!]
    \centering
    \includegraphics[width=0.3\linewidth]{imgs/multi_mnist_original.png}
    \includegraphics[width=0.3\linewidth]{imgs/multi_mnist_0_orig.png}
    \includegraphics[width=0.3\linewidth]{imgs/multi_mnist_0_infocam.png}
    \caption{Visualisation of comparison between CAM and infoCAM+ for the multi-MNIST dataset. Left: original input images, each with one or two digits in the left and/or right; Middle: Localisation of digit 0 with the original CAM; Right: Localisation of digit 0 with infoCAM. }
    \label{fig:multi_mnist}
\end{figure}
\end{comment}
So far, we have shown the results of localisation from a multi-class classification problem. We further extend our experiments on localisation to multi-label classification problems.

A softmax function is a generalisation of its binary case, a sigmoid function. Therefore, we can apply infoCAM to each label for a multi-label classification problem, which is a collection of binary classification tasks.
%given that infoCAM presents a pleasing performance in localising a single object within an image, we expect infoCAM can also perform well for WSOL where multiple objects exist in an image. 

For the experiment, we construct a double-digit MNIST dataset where each image contains up to two digits randomly sampled from the original MNIST dataset~\cite{lecun2010mnist}. We locate one digit on the left-side, and the other on the right-side. Some of the images only contain a single digit. For each side, we first decide whether to include a digit from a Bernoulli distribution with mean of 0.7. Then each digit is randomly sampled from a uniform distribution. However, we remove the images that contain no digits. 
Random samples from the double-digit MNIST are shown in \autoref{fig:sub-first}.

\begin{table}[t!]
\centering
\begin{tabular}{c l l l l l l l l l l}
\toprule
\multirow{2}{*}{\makecell{Type}} & \multicolumn{10}{c}{Digit Classification Accuracy (\%)} \\
& 0 & 1 & 2 & 3 & 4 & 5 & 6 & 7 & 8 & 9 \\
\midrule
%\cmidrule(lr{1em}){2-11}
sigmoid & 1.00 & 0.84 & 0.86 & 0.94 & 0.89 & 0.87 & 0.87 & 0.86 & 1.00 & 1.00 \\
PC-sigmoid & 1.00 & 1.00 & 1.00 & 1.00 & 1.00 & 1.00 & 1.00 & 1.00 & 1.00 & 1.00 \\
\bottomrule
\end{tabular}
\caption{Comparison between the classification accuracy results with sigmoid and PC-sigmoid on the double-digit MNIST dataset. }
\label{tbl:sig_pc_sig_acc}
\end{table}

We first compare the classification accuracy results between using the original sigmoid and PC-sigmoid. 
%We consider the threshold of classifying the input to be positive as 0.5. That is, for both sigmoid and PC-sigmoid, if the output value is larger than 0.5, then we think the input to be in the class. Pleasingly, 
As shown in \autoref{tbl:sig_pc_sig_acc}, PC-sigmoid increases the classification accuracy for each digit type on the test set.
InfoCAM improves the localisation accuracy for the WSOL task as well. CAM achieves the localisation accuracy of 91\%. InfoCAM enhances the localisation accuracy to 98\%. Qualitative visualisations are displayed in \autoref{fig:multi_mnist}. We aim to preserve the regions of an image that are most relevant to a digit, and erase all the other regions. From the visualisation, one can see that infoCAM localises digits more accurately than CAM.

\begin{comment}
Experimental results with multi-digit MNIST~\cite{lecun2010mnist} and a subset of COCO datasets~\cite{lin2014microsoft} verified our expectation. Qualitative visualisations are displayed in \autoref{fig:multi_mnist} and \autoref{fig:multi_coco}. As to quantitative results, for multi-digit MNIST, the localisation accuracy with the original CAM and infoCAM are 91\% and 98\% respectively; for COCO, the localisation accuracy with the original CAM and infoCAM are 89\% and 94\% respective. For detailed experimental settings and results, please refer to the appendix. 
\end{comment}

\section{Conclusion}
We have shown the connection between mutual information estimators and neural network classifiers through the variational form of mutual information. The connection explains the rationale behind the use of sigmoid,  softmax and cross-entropy from an information-theoretic perspective. The connection also brings a new insight to understand neural network classifiers. There exists previous work that called the negative log-likelihood (NLL) loss as maximum mutual information estimation~\cite{bahl1986maximum,lecun2006tutorial}. Despite this naming similarity, that work does not show the relationship between softmax and mutual information that we have shown here. 

The connection between neural network classifiers and mutual information evaluators provides more than an alternative view on neural network classifiers. Via converting neural network classifiers to mutual information estimators, we receive two positive consequences for practical applications. First, we improve the classification accuracy, in particular when the datasets are unbalanced. The new mutual information estimators even outperform the prior state-of-the-art neural network classifiers. Second, using the pointwise mutual information between the inputs and labels, we can locate the objects within images more precisely. We also provide a more information-theoretic interpretation of class activation maps. We believe that this opens new ways to understand how neural network classifiers work and improve their performance. 

\begin{comment}
We have shown the connection between mutual information and softmax classifier through the variational form of mutual information. The connection explains the rationale behind the use of softmax cross-entropy from an information theoretic perspective, which brings a new insight to understand such classifiers.
%Although we could not find significant difference between softmax and PC-softmax, it can be fruitful to some extend since most of existing classification models enjoy the connection without changing the model architecture.
There exists previous work that called the negative log-likelihood (NLL) loss as maximum mutual information estimation~\cite{bahl1986maximum,lecun2006tutorial}. Despite this naming similarity, that work does not show the relationship between softmax and mutual information that we have shown here. 

We utilise the connection between classification and mutual information to improve the weakly-supervised object localisation task. To this end, we propose a new way to compute the classification activation map, which is based on the difference between PMIs. The experimental results show the practicality of the information theoretic approach. We believe that this opens new ways to understand and interpret how neural network classifiers work.
% We may add a few sentences about future work
\end{comment}

{\small
\bibliographystyle{ieee}
\bibliography{egbib}

\begin{thebibliography}{10}\itemsep=-1pt

\bibitem{tiny_imagenet}
Tiny imagenet visual recognition challenge.
\newblock \url{https://tiny-imagenet.herokuapp.com/}.
\newblock Accessed: 2019-11-03.

\bibitem{bahl1986maximum}
Lalit~R Bahl, Peter~F Brown, Peter~V De~Souza, and Robert~L Mercer.
\newblock Maximum mutual information estimation of hidden markov model
  parameters for speech recognition.
\newblock In {\em Proc. ICASSP}, volume~86, pages 49--52, 1986.

\bibitem{barber2003algorithm}
David Barber and Felix~V Agakov.
\newblock The im algorithm: a variational approach to information maximization.
\newblock In {\em Advances in Neural Information Processing Systems}, page
  None, 2003.

\bibitem{belghazi2018mutual}
Mohamed~Ishmael Belghazi, Aristide Baratin, Sai Rajeshwar, Sherjil Ozair,
  Yoshua Bengio, Aaron Courville, and Devon Hjelm.
\newblock Mutual information neural estimation.
\newblock In {\em International Conference on Machine Learning}, pages
  531--540, 2018.

\bibitem{choe2018improved}
Junsuk Choe, Joo~Hyun Park, and Hyunjung Shim.
\newblock Improved techniques for weakly-supervised object localization.
\newblock {\em arXiv preprint arXiv:1802.07888}, 2018.

\bibitem{choe2019attention}
Junsuk Choe and Hyunjung Shim.
\newblock Attention-based dropout layer for weakly supervised object
  localization.
\newblock In {\em Proceedings of the IEEE Conference on Computer Vision and
  Pattern Recognition}, pages 2219--2228, 2019.

\bibitem{cui2018large}
Yin Cui, Yang Song, Chen Sun, Andrew Howard, and Serge Belongie.
\newblock Large scale fine-grained categorization and domain-specific transfer
  learning.
\newblock In {\em Proceedings of the IEEE Conference on Computer Vision and
  Pattern Recognition}, pages 4109--4118, 2018.

\bibitem{dubey2018pairwise}
Abhimanyu Dubey, Otkrist Gupta, Pei Guo, Ramesh Raskar, Ryan Farrell, and
  Nikhil Naik.
\newblock Pairwise confusion for fine-grained visual classification.
\newblock In {\em Proceedings of the European Conference on Computer Vision
  (ECCV)}, pages 70--86, 2018.

\bibitem{duchi2011adaptive}
John Duchi, Elad Hazan, and Yoram Singer.
\newblock Adaptive subgradient methods for online learning and stochastic
  optimization.
\newblock {\em Journal of Machine Learning Research}, 12(7), 2011.

\bibitem{geer2000empirical}
Sara~A Geer and Sara van~de Geer.
\newblock {\em Empirical Processes in M-estimation}, volume~6.
\newblock Cambridge University Press, 2000.

\bibitem{he2016deep}
Kaiming He, Xiangyu Zhang, Shaoqing Ren, and Jian Sun.
\newblock Deep residual learning for image recognition.
\newblock In {\em Proceedings of the IEEE Conference on Computer Vision and
  Pattern Recognition}, pages 770--778, 2016.

\bibitem{hjelm2018learning}
R~Devon Hjelm, Alex Fedorov, Samuel Lavoie-Marchildon, Karan Grewal, Phil
  Bachman, Adam Trischler, and Yoshua Bengio.
\newblock Learning deep representations by mutual information estimation and
  maximization.
\newblock In {\em International Conference on Learning Representation}, 2019.

\bibitem{hornik1989multilayer}
Kurt Hornik, Maxwell Stinchcombe, and Halbert White.
\newblock Multilayer feedforward networks are universal approximators.
\newblock {\em Neural Networks}, 2(5):359--366, 1989.

\bibitem{kingma2014adam}
Diederik~P Kingma and Jimmy Ba.
\newblock Adam: A method for stochastic optimization.
\newblock {\em arXiv preprint arXiv:1412.6980}, 2014.

\bibitem{krizhevsky2012imagenet}
Alex Krizhevsky, Ilya Sutskever, and Geoffrey~E Hinton.
\newblock Imagenet classification with deep convolutional neural networks.
\newblock In {\em Advances in Neural Information Processing Systems}, pages
  1097--1105, 2012.

\bibitem{lecun2015deep}
Yann LeCun, Yoshua Bengio, and Geoffrey Hinton.
\newblock Deep learning.
\newblock {\em Nature}, 521(7553):436, 2015.

\bibitem{lecun2006tutorial}
Yann Lecun, Sumit Chopra, Raia Hadsell, Marc~Aurelio Ranzato, and Fu~Jie Huang.
\newblock A tutorial on energy-based learning.
\newblock In {\em Predicting Structured Data}. MIT Press, 2006.

\bibitem{lecun2010mnist}
Yann LeCun, Corinna Cortes, and CJ Burges.
\newblock Mnist handwritten digit database.
\newblock 2010.

\bibitem{lin2014network}
Min Lin, Qiang Chen, and Shuicheng Yan.
\newblock Network in network.
\newblock In {\em International Conference on Learning Representation}, 2014.

\bibitem{maas2013rectifier}
Andrew~L Maas, Awni~Y Hannun, and Andrew~Y Ng.
\newblock Rectifier nonlinearities improve neural network acoustic models.
\newblock In {\em Proc. ICML}, volume~30, page~3, 2013.

\bibitem{mcallester2018formal}
David McAllester and Karl Statos.
\newblock Formal limitations on the measurement of mutual information.
\newblock {\em arXiv preprint arXiv:1811.04251}, 2018.

\bibitem{nair2010rectified}
Vinod Nair and Geoffrey~E Hinton.
\newblock Rectified linear units improve restricted boltzmann machines.
\newblock In {\em Proc. ICML}, 2010.

\bibitem{odena2017conditional}
Augustus Odena, Christopher Olah, and Jonathon Shlens.
\newblock Conditional image synthesis with auxiliary classifier gans.
\newblock In {\em Proceedings of the 34th International Conference on Machine
  Learning-Volume 70}, pages 2642--2651. JMLR. org, 2017.

\bibitem{paszke2017automatic}
Adam Paszke, Sam Gross, Soumith Chintala, Gregory Chanan, Edward Yang, Zachary
  DeVito, Zeming Lin, Alban Desmaison, Luca Antiga, and Adam Lerer.
\newblock Automatic differentiation in pytorch.
\newblock In {\em Proceedings of Neural Information Processing Systems}, 2017.

\bibitem{poole2019variational}
Ben Poole, Sherjil Ozair, Aaron van~den Oord, Alexander~A Alemi, and George
  Tucker.
\newblock On variational bounds of mutual information.
\newblock In {\em International Conference on Machine Learning}, 2019.

\bibitem{srivastava2014dropout}
Nitish Srivastava, Geoffrey Hinton, Alex Krizhevsky, Ilya Sutskever, and Ruslan
  Salakhutdinov.
\newblock Dropout: a simple way to prevent neural networks from overfitting.
\newblock {\em Journal of Machine Learning Research}, 15(1):1929--1958, 2014.

\bibitem{szegedy2015going}
Christian Szegedy, Wei Liu, Yangqing Jia, Pierre Sermanet, Scott Reed, Dragomir
  Anguelov, Dumitru Erhan, Vincent Vanhoucke, and Andrew Rabinovich.
\newblock Going deeper with convolutions.
\newblock In {\em Proceedings of the IEEE Conference on Computer Vision and
  Pattern Recognition}, pages 1--9, 2015.

\bibitem{szegedy2016rethinking}
Christian Szegedy, Vincent Vanhoucke, Sergey Ioffe, Jon Shlens, and Zbigniew
  Wojna.
\newblock Rethinking the inception architecture for computer vision.
\newblock In {\em Proceedings of the IEEE Conference on Computer Vision and
  Pattern Recognition}, pages 2818--2826, 2016.

\bibitem{wah2011caltech}
Catherine Wah, Steve Branson, Peter Welinder, Pietro Perona, and Serge
  Belongie.
\newblock The caltech-ucsd birds-200-2011 dataset.
\newblock 2011.

\bibitem{wang2019camdrop}
Hongjun Wang, Guangrun Wang, Guanbin Li, and Liang Lin.
\newblock Camdrop: A new explanation of dropout and a guided regularization
  method for deep neural networks.
\newblock In {\em International Conference on Information and Knowledge
  Management (CIKM)}, pages 2219--2228, 2019.

\bibitem{zhao2017infovae}
Shengjia Zhao, Jiaming Song, and Stefano Ermon.
\newblock Infovae: Information maximizing variational autoencoders.
\newblock {\em arXiv preprint arXiv:1706.02262}, 2017.

\bibitem{zhou2016learning}
Bolei Zhou, Aditya Khosla, Agata Lapedriza, Aude Oliva, and Antonio Torralba.
\newblock Learning deep features for discriminative localization.
\newblock In {\em Proceedings of the IEEE Conference on Computer Vision and
  Pattern Recognition}, pages 2921--2929, 2016.

\end{thebibliography}
}

\newpage
\onecolumn
\appendix
\begin{center}
{\LARGE\bfseries Supplementary Materials}

{\Large Rethinking Softmax with Cross-Entropy: \\
Neural Network Classifier as \\ Mutual Information Estimator}
\end{center}

\section{Network Architectures}
In this section, we illustrate neural network architectures that have been utilised in the previous experiments. \autoref{fig:softmax_mi_architecture} demonstrates the architecture of the softmax mutual information neural estimator in \autoref{sec:mi_class_exp}. \autoref{fig:mnist_class_architecture} demonstrates the architecture of the network that are utilised to show PC-softmax leads to higher classification accuracy on the unbalanced MNIST dataset as in \autoref{sec:mi_class_exp}. We explain in \autoref{sec:exp} on how to convert the VGG16 architecture to the VGG16-GAP architecture. Such VGG16-GAP is used in infoCAM. We illustrate in \autoref{fig:vgg_gap_architecture} on how to convert the former to the latter architecture. For ResNet50 and Inception-V3, the architectures are identical to \cite{he2016deep} and \cite{szegedy2016rethinking}.

For more detailed information, please refer to the actual implementation, which we plan to make public. 

\begin{figure}[!htbp]
    \centering
    \includegraphics[width=0.6\linewidth]{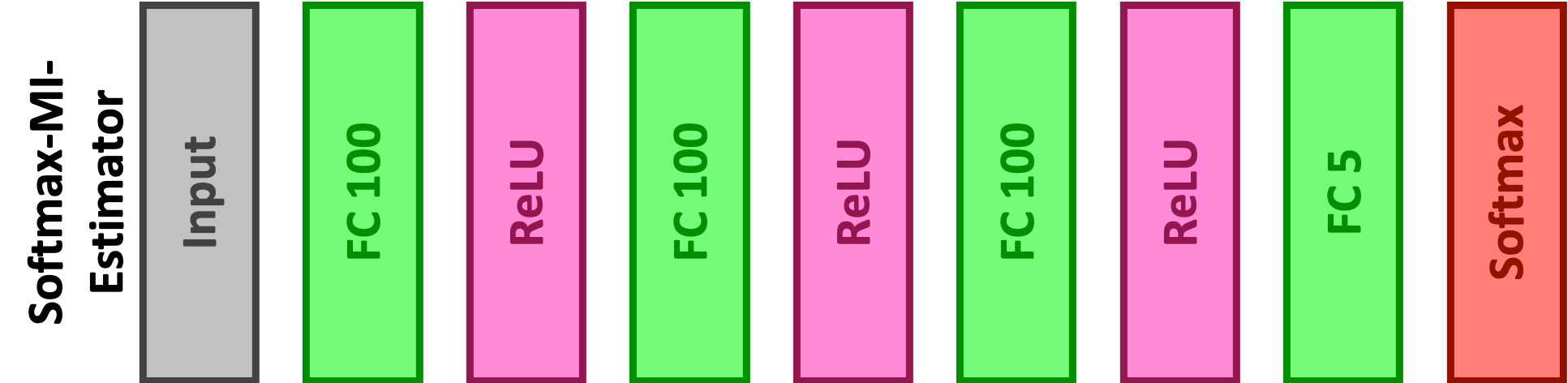}
    \caption{The neural network architecture of the softmax mutual information estimator. The softmax in the last layer can be either the traditional or the PC one. }
    \label{fig:softmax_mi_architecture}
\end{figure}

\begin{figure}[!htbp]
    \centering
    \includegraphics[width=0.6\linewidth]{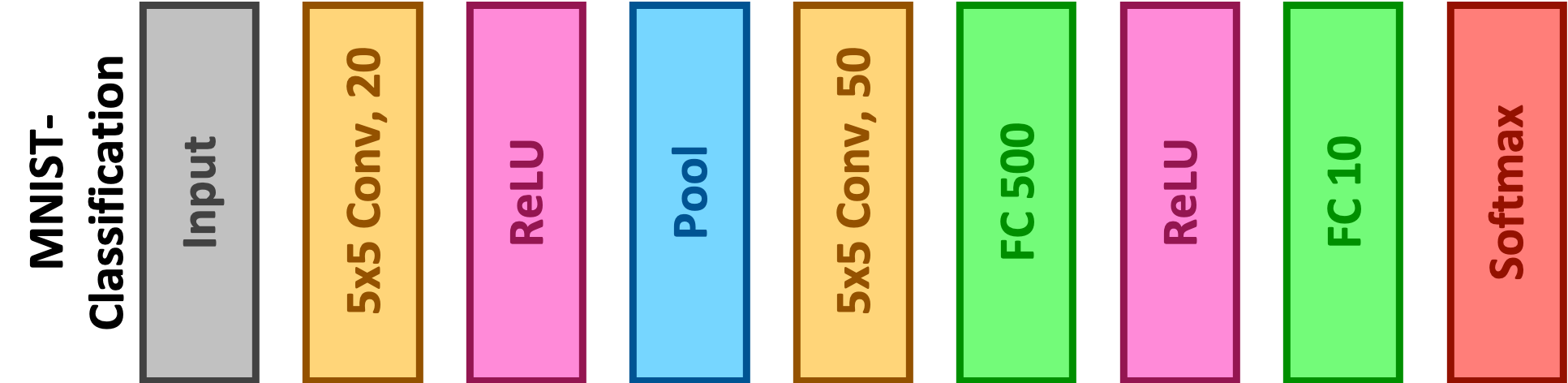}
    \caption{The neural network architecture that is utilised to show PC-softmax leads to higher classification accuracy on the unbalanced MNIST dataset. The softmax in the last layer can be either the traditional or the PC one. }
    \label{fig:mnist_class_architecture}
\end{figure}

\begin{figure}[!htbp]
    \centering
    \includegraphics[width=0.6\linewidth]{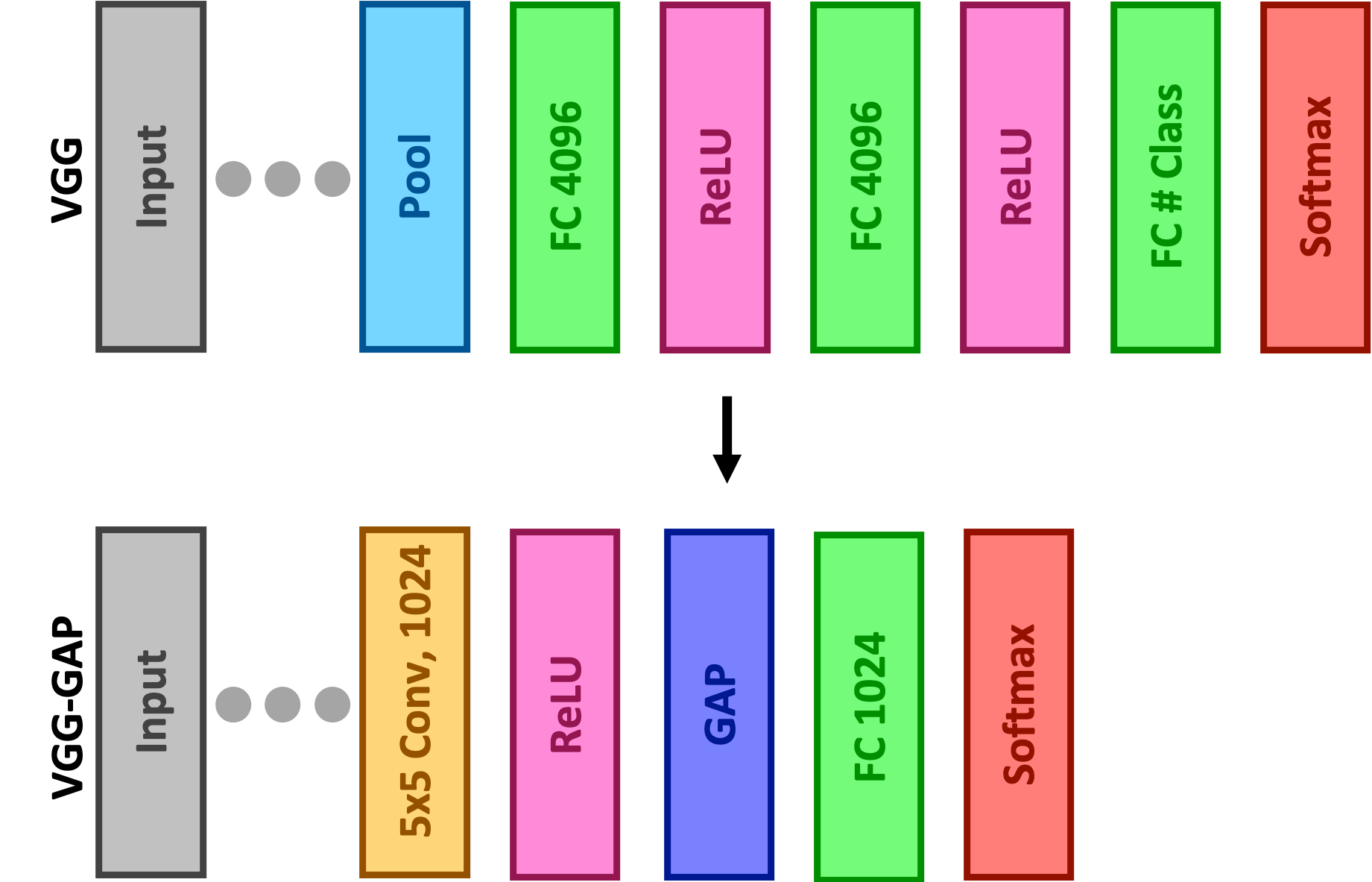}
    \caption{Illustration on the conversion from VGG16 to VGG16-GAP. }
    \label{fig:vgg_gap_architecture}
\end{figure}

\section{Visualisation of Data Distributions}
We show both theoretically and experimentally in \autoref{sec:mi_class_exp} and \autoref{sec:nn_as_mi_evaluators} that neural network classifiers can be considered as mutual information estimators. In this section, we provide visualisation on the distributions of the data that are used to test the effectiveness of the softmax-based mutual information estimator. In such visualisation as \autoref{fig:synthetic_data_distribution}, data points are stratified subsets of the test datasets, so that it can reflect the dataset imbalance. Since it is impossible to visualise data whose dimension is greater or equal to three, we apply principle component analysis (PCA) to reduce the dimension to two. Furthermore, data of different class labels become more distinguishable as dimension increases. This can account for the reason why classification accuracy increases as the dimension rises. 

\begin{figure}[!htbp]
    \centering
    \includegraphics[width=\linewidth]{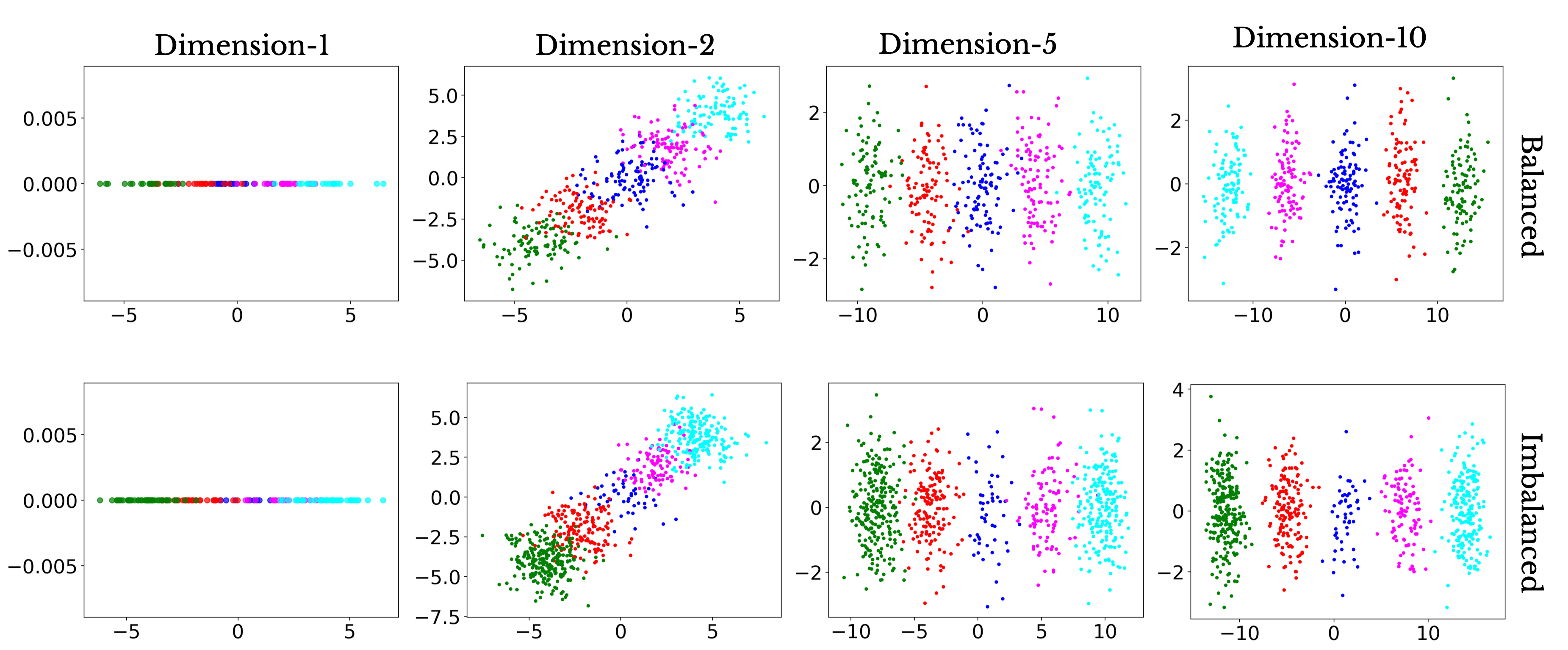}
    \caption{Illustration of the synthetic dataset for evaluating the softmax-based mutual information estimator. For data whose dimension is greater or equal to three, the visualisation is on the results of PCA. The same colour represents the identical class. }
    \label{fig:synthetic_data_distribution}
\end{figure}

\section{Further Result}
In this section, we present some further results on localisation and classification.

\subsection{Localisation and Classification Result}

\autoref{table:cam_info_cam_full} is a reproduction of main result with the classification results. 
Note that the classification performances of CAM and infoCAM is the same since we do not modify the training objective of infoCAM. The result can be used to understand the effect of ADL on the classification task.

\begin{table*}
\centering
\begin{tabular}{c l c c c c}
\toprule
&  & 
\makecell{GT \\ Loc. (\%)} &
\makecell{Top-1 \\ Loc. (\%)} & 
\makecell{Top-1 \\ Cls (\%)} & 
\makecell{Top-5 \\ Cls (\%)}  \\ 
\midrule
\multirow{6}{*}{\makecell{VGG-\\16-\\GAP}} & CAM& 42.49 & 31.38 & 73.97 & 91.83   \\
& CAM (ADL) & 71.59 & 53.01 & 71.05 & 90.20 \\
& {infoCAM} & {52.96} & {39.79} & - & - \\
& {infoCAM} (ADL) & {73.35} & {53.80} & - & - \\
& {infoCAM+} & {59.43} & {44.40} & - & - \\
& {infoCAM+} (ADL) & \textbf{75.89} & {54.35} & - & - \\ \midrule
\multirow{6}{*}{\makecell{ResNet-\\50}} & CAM& 61.66 & 50.84 & 80.54 & 94.09 \\
& CAM (ADL) & 57.83 & 46.56 & 79.22 & 94.02 \\
& {infoCAM} &  {64.78} & {53.22} & - & -  \\
& {infoCAM} (ADL) & {67.75} & {54.71} & - & -  \\
& {infoCAM+} &  {68.99} & \textbf{55.83} & - & -  \\
& {infoCAM+} (ADL) & {69.63} & {55.20} & - & -  \\
\bottomrule
\end{tabular}

\caption{Evaluation results of CAM and infoCAM on CUB-2011-200. Note that the classification accuracy of infoCAM is the same as those of CAM. InfoCAM always outperforms CAM on localisation of objects under the same model architecture.}
\label{table:cam_info_cam_full}
\end{table*}

\begin{table*}
\centering

\centering
\begin{tabular}{c l c c c c}
\toprule
 &  & 
\makecell{GT \\ Loc. (\%)} &
\makecell{Top-1 \\ Loc. (\%)} & 
\makecell{Top-1 \\ Cls (\%)} & 
\makecell{Top-5 \\ Cls (\%)}  \\ 
\midrule
\multirow{6}{*}{\makecell{VGG-\\16-\\GAP}} & CAM& 53.49 & 33.48 & 55.25 & 79.19  \\
& CAM (ADL) & 52.75 & 32.26 & 52.48 & 78.75 \\
& {infoCAM} & {55.50} & {34.27} & - & -\\
& {infoCAM} (ADL) & {53.95} & {33.05} & - & - \\ 
& {infoCAM+} & {55.25} & {34.27} & - & -\\
& {infoCAM+} (ADL) & {53.91} & {32.94} & - & - \\ \midrule
\multirow{6}{*}{\makecell{ResNet-\\50}} & CAM& 54.56 & 40.55 & 66.45 & 86.22 \\
& CAM (ADL) & 52.66 & 36.88 & 63.21 & 83.47 \\
& {infoCAM}& \textbf{57.79} & \textbf{43.34} & - & - \\
& {infoCAM} (ADL) & {54.18} & {37.79} & - & - \\
& {infoCAM+}& {57.71} & {43.07} & - & - \\
& {infoCAM+} (ADL) & {53.70} & {37.71} & - & - \\
\bottomrule
\end{tabular}

\caption{Evaluation results of CAM and infoCAM on Tiny-ImageNet. Note that the classification accuracy of infoCAM is the same as those of CAM. InfoCAM always outperforms CAM on localisation of objects under the same model architecture.}
\label{table:cam_info_cam_full}
\end{table*}

\subsection{Ablation Study}

\autoref{tbl:ablation} shows the result of ablation study. We have tested the importance of three features: 1) ADL, 2) region parameter $R$ and 3) the second subtraction term in \autoref{eq:info_cam_correct}. To combine the result in the main text, the result suggests that both region parameter and subtraction term are necessary to increase the performance of localisation. The choice of ADL depends on the dataset. We conjecture that ADL is inappropriate to apply Tiny-ImageNet since the removal of any part of tiny image, which is what ADL does during training, affects the performance of the localisation to compare with its application to relatively large images.

\begin{table}[t!]
\centering

\begin{subtable}[t]{\linewidth}
\centering
\begin{tabular}{c c c l l}
\toprule
ADL & \makecell{$R$} & 
\makecell{Subtraction \\ Term} &
\makecell{GT Loc. (\%)} & 
\makecell{Top-1 \\ Loc. (\%)} \\ 
\midrule
\multirow{3}{*}{N} & N& N& 42.49& 31.38 \\
 & N& Y& 47.59 $\uparrow$& 35.01 $\uparrow$ \\ 
 & Y& N& 53.40 $\uparrow$& 40.19 $\uparrow$ \\ \midrule
\multirow{3}{*}{Y} & N& N& 71.59& 53.01 \\
 & N& Y& 75.78 $\uparrow$& 54.28 $\uparrow$ \\ 
 & Y& N& 73.56 $\uparrow$& 53.94 $\uparrow$ \\
\bottomrule
\end{tabular}
\caption{Localisation results on CUB-200-2011 with VGG-GAP.}
\label{tbl:ablation}
\end{subtable}

\vspace{1em}

\begin{subtable}[t]{\linewidth}
\centering
\begin{tabular}{c c c l l}
\toprule
ADL & \makecell{$R$} & 
\makecell{Subtraction \\ Term} &
\makecell{GT Loc. (\%)} & 
\makecell{Top-1 \\ Loc. (\%)} \\ 
\midrule
\multirow{3}{*}{N} & N& N& 54.56& 40.55 \\
 & N& Y& 54.29 $\downarrow$& 40.51 $\downarrow$ \\ 
 & Y& N& 57.73 $\uparrow$& 43.34 $\uparrow$ \\ \midrule
\multirow{3}{*}{Y} & N& N& 52.66& 36.88 \\
 & N& Y& 52.52 $\downarrow$& 37.08 $\uparrow$ \\ 
 & Y& N& 54.15 $\uparrow$& 37.76 $\uparrow$ \\
\bottomrule
\end{tabular}
\caption{Localisation results on CUB-200-2011 with ResNet50.}
\end{subtable}

\caption{Ablation study results on the importance of the region parameter $R$ and the subtraction term within the formulation of infoCAM.  Y and N indicates the use of corresponding feature. Arrows indicates the relative performance against the case where both features are not used.}
\label{table:ablation_study}
\end{table}

\subsection{Localisation Examples from Tiny-ImageNet}
We present examples from the Tiny-ImageNet dataset in \autoref{fig:cam-illustration-butterfly}. Such examples show the infoCAM draws tighter bound toward target objects.
\begin{comment}
We present more examples from both datasets in \autoref{fig:cam-illustration-butterfly}. Again, we can find the localisation with CAM focuses on the head area of birds on CUB-200-2011 dataset. The examples from Tiny-ImageNet show the infoCAM draws tighter bound toward target objects.
\end{comment}

\begin{figure*}[t!]
    \centering
    \includegraphics[width=\linewidth]{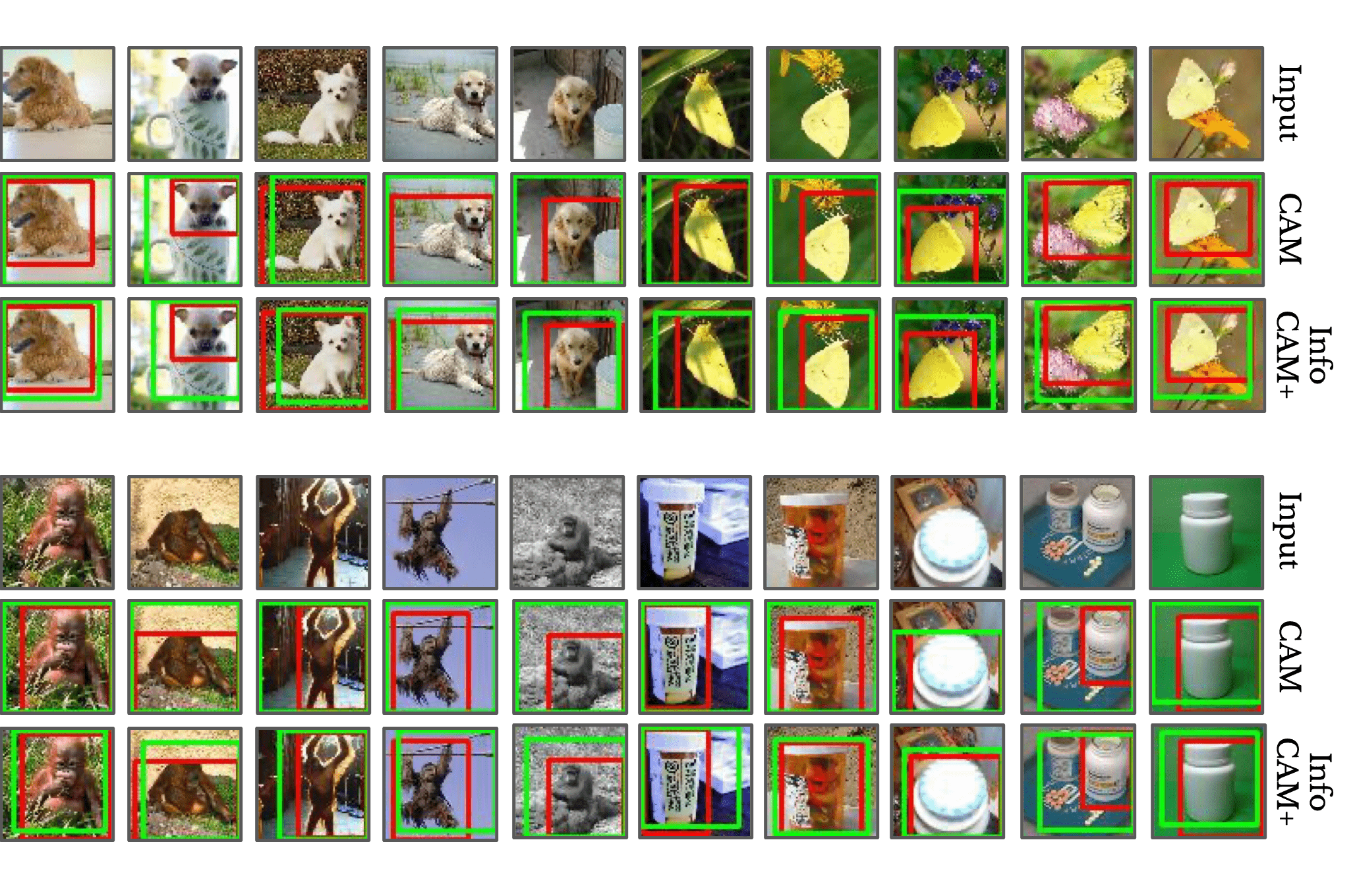}
    \caption{Visualisation of localisation with ResNet50 on CUB-200-2011 and TinyImageNet, without the assistance of ADL. The images in the second row are generated from the original CAM approach and the ones in the third row correspond to infoCAM. The red and green bounding boxes are ground truth and estimations, respectively. }
    \label{fig:cam-illustration-butterfly}
\end{figure*}

\begin{comment}
\section{Experimental Details about Multi-Object Localisation}
We employ the multi-MNIST and COCO datasets to study the performance of infoCAM for localising multiple objects within an image. We first describe the settings of multi-MNIST. Each image of the dataset consists of two digits, with one in the left and one in the right. The locations of the two digits are symmetric. As to the settings of COCO, instead of using the giant original dataset, we only use images of persons and cats. That is, there are three classes of images: persons only, cats only and persons and cats. For each class, there are 90 images for training and 10 images for testing. We use the very small subset of the original dataset to reduce the requirement of computational resources.
\end{comment}

\section{Proofs}
\label{sec:proofs}
In this section, we provide rigorous proofs of \autoref{thm:equality} and \autoref{thm:softmax_im}. The structure of proof is similar to the proof used in \cite{belghazi2018mutual}. We assume the input space $\Omega = \X \times Y$ being a compact domain of $\mathcal{R}^{d}$, where all measures are Lebesgue and are absolutely continuous. We restrict neural networks to produce a single continuous output, denoted as $n(\x)_y$. We restate the two theorems for quick reference.

\textbf{Theorem 1.} \textit{Let $f_\phi(\x,y)$ be $n(\x)_y$. Minimising the cross-entropy loss of softmax-normalised neural network outputs is equivalent to maximising \autoref{eq:mi_f_low_bound}, \ie the lower bound of mutual information, under the uniform label distribution. That is, if the dataset is balanced, then training a neural network via minimising cross-entropy with softmax equals enhancing a estimator toward more accurately evaluating the mutual information between data and label.}

\textbf{Theorem 2.} \textit{The mutual information between two random variable $X$ and $Y$ can be obtained via the infimum of cross-entropy with PC-softmax in \autoref{eq:prob_cor_softmax}. Such an evaluation is strongly consistent. }

The proof technique that we have used to prove Theorem 2 is similar to the one used in \cite{belghazi2018mutual}.

\begin{lem}
Let $\eta > 0$. There exists a family of neural network functions $\nphi$ with parameter $\mathbf{\phi}$ in some compact domain such that
\begin{align}
    |\mii(\X;Y) - \miiphi(\X;Y)| \le \eta, 
\end{align}
where 
\begin{align}
    \miiphi(\X;Y) = \underset{\mathbf{\phi}}{\sup} \ \ejxy \big[ \nphi \big] - \ex \log \ey [\exp(\nphi)_y]. 
\end{align}
\end{lem}
\begin{proof}
Let $\nphi^{\ast}(\X,Y) = \PMI(\X,Y) = \log \frac{P(\X, Y)}{P(\X)P(Y)}$. We then have: 
\begin{align}
    \ejxy [\nphi^{\ast}(\x)_y] = \mii(\X;Y) \quad \text{and} \quad \emxy[ \exp{(\nphi^{\ast}(\x)_y)} ] = 1. 
\end{align}
Then, for neural network $\nphi$, the gap $\mii(\X;Y) - \miiphi(\X;Y)$: 
\begin{align}
    \mii(\X;Y) - \miiphi(\X;Y) 
    &= \ejxy [ \nphi^{\ast}(\X,Y) - \nphi(\X,Y) ] + \ex \log \ey [\exp(\nphi)_y] \notag \\
    &\le \ejxy [ \nphi^{\ast}(\X,Y) - \nphi(\X,Y) ] + \log \emxy[\exp(\nphi(\x)_y)] \notag \\
    &\le \ejxy [ \nphi^{\ast}(\X,Y) - \nphi(\X,Y) ] \notag\\
    & \quad + \emxy[\exp(\nphi(\x)_y) - \exp{(\nphi^{\ast}(\x)_y)} ].  \label{eq:mi_diff}
\end{align}
\autoref{eq:mi_diff} is positive since the neural mutual information estimator evaluates a lower bound. The equation uses Jensen's inequality and the inequality $\log x \le x - 1$. 

We assume $\eta > 0$ and consider $\nphi^{\ast}(\x)_y$ is bounded by a positive constant $M$. Via the universal approximation theorem \cite{hornik1989multilayer}, there exists $\nphi(\x)_y \le M$ such that
\begin{align}
    \ejxy | \nphi^{\ast}(\X,Y) - \nphi(\X,Y) | \le \frac{\eta}{2} \quad \text{and} \quad \emxy|\nphi(\x)_y - \nphi^{\ast}(\x)_y | \le \frac{\eta}{2} \exp{(-M)}. 
    \label{eq:mi_diff_up}
\end{align}

By utilising that $\exp$ is Lipschitz continuous with constant $\exp(M)$ over $(-\infty, M]$, we have
\begin{align}
    \emxy|\exp(\nphi(\x)_y) - \exp(\nphi^{\ast}(\x)_y) | \le \exp(M) \cdot  \emxy|\nphi(\x)_y - \nphi^{\ast}(\x)_y | \le \frac{\eta}{2}.
    \label{eq:mi_diff_lips}
\end{align}

Combining \autoref{eq:mi_diff}, \autoref{eq:mi_diff_up} and \autoref{eq:mi_diff_lips}, we then obtain
\begin{align}
    |\mii(\X;Y) - \miiphi(\X;Y)| &\le \ejxy | \nphi^{\ast}(\X,Y) - \nphi(\X,Y) | \notag \\
    &\quad + \emxy|\exp(\nphi(\x)_y) - \exp(\nphi^{\ast}(\x)_y) | \notag \\
    &= \frac{\eta}{2} + \frac{\eta}{2} = \eta. 
\end{align}
\end{proof}

\begin{lem}
Let $\eta > 0$. Given a family of neural networks $\nphi$ with parameter $\mathbf{\phi}$ in some compact domain, there exists $N \in \mathbb{N}$ such that
\begin{align}
    \forall n \ge N, \text{Pr}\big( | \miin(\X;Y) - \nphi(\X;Y) | \le \eta \big) = 1. 
\end{align}
\begin{proof}
We start by employing the triangular inequality: 
\begin{gather}
    | \miin(\X;Y) - \miiphi(\X;Y) | \notag \\
    \le \underset{\mathbf{\phi}}{\sup} \ |\ejxy[\nphi^{\ast}(\X,Y)] - \ejxyn[\nphi^{\ast}(\X,Y)]| \notag \\ 
    + \underset{\mathbf{\phi}}{\sup} \ |\ex \log \ey [\exp(\nphi)_y] - \exn \log \eyn [\exp(\nphi)_y]| \label{eq:mi_diff_n_phi}
\end{gather}

We have stated previously that neural network $\nphi$ is bounded by $M$, \ie $\nphi(\x)_y \le M$. Using the fact that $\log$ is Lipschitz continuous with constant $\exp(M)$ over the interval $[\exp(-M), \exp(M)]$. We have
\begin{equation}
    | \log \ey [\exp(\nphi)_y] - \log \eyn [\exp(\nphi)_y] | \le \exp(M) \cdot |\ey [\exp(\nphi)_y] - \eyn [\exp(\nphi)_y]|
\end{equation}

Using the uniform law of large numbers \cite{geer2000empirical}, we can choose $N \in \mathbb{N}$ such that for $\forall n \ge N$ and with probability one
\begin{equation}
\underset{\mathbf{\phi}}{\sup} \ |\ey [\exp(\nphi)_y] - \eyn [\exp(\nphi)_y]| \le \frac{\eta}{4} \exp(-M). 
\end{equation}
That is, 
\begin{align}
    | \log \ey [\exp(\nphi)_y] - \log \eyn [\exp(\nphi)_y] | \le \frac{\eta}{4}
\end{align}
Therefore, using the triangle inequality we can rewrite \autoref{eq:mi_diff_n_phi} as: 
\begin{gather}
    | \miin(\X;Y) - \miiphi(\X;Y) |
    \le \underset{\mathbf{\phi}}{\sup} \ |\ejxy[\nphi^{\ast}(\X,Y)] - \ejxyn[\nphi^{\ast}(\X,Y)]| \notag \\ 
    + \underset{\mathbf{\phi}}{\sup} \ |\ex \log \ey [\exp(\nphi)_y] - \exn \log \ey [\exp(\nphi)_y]| + \frac{\eta}{4}. 
    \label{eq:mi_diff_add_const}
\end{gather}

Using the uniform law of large numbers again, we can choose $N \in \mathbb{N}$ such that for $\forall n \ge N$ and with probability one
\begin{align}
    \underset{\mathbf{\phi}}{\sup} \ |\ex \log \ey [\exp(\nphi)_y] - \exn \log \ey [\exp(\nphi)_y]| \le \frac{\eta}{4} 
    \label{eq:margin_less}
\end{align}
and: 
\begin{align}
    \underset{\mathbf{\phi}}{\sup} \ |\ejxy[\nphi^{\ast}(\X,Y)] - \ejxyn[\nphi^{\ast}(\X,Y)]| \le \frac{\eta}{2}. 
    \label{eq:join_less}
\end{align}

Combining \autoref{eq:mi_diff_add_const}, \autoref{eq:margin_less} and \autoref{eq:join_less} leads to
\begin{equation}
    | \miin(\X;Y) - \underset{\phi}{\sup} \ \miiphi(\X;Y) | \le \frac{\eta}{2} + \frac{\eta}{4} + \frac{\eta}{4} = \eta. 
\end{equation}
\end{proof}
\end{lem}

Now, combining the above two lemmas, we prove that our mutual information evaluator is strongly consistent. 
\begin{proof}
Using the triangular inequality, we have
\begin{align}
    |\mii(\X;Y) - \miin(\X;Y)| \le |\mii(\X;Y) - \miiphi(\X;Y)| + |\miin(\X;Y) - \nphi(\X;Y)| \le \epsilon. 
\end{align}
\end{proof}

\end{document}